%% file: progstruc.tex
\documentclass{article}

\usepackage{microtype}
\usepackage{graphicx}
\usepackage{subfigure}
\usepackage{booktabs}
\usepackage{algorithm}
\usepackage{algorithmic}
\usepackage{amsmath}
\usepackage{adjustbox}
\usepackage{natbib}

\usepackage{hyperref}

\usepackage[accepted]{icml2018}

\icmltitlerunning{Learning Neurosymbolic Generative Models via Program Synthesis}

\input{macros}

\begin{document}

\twocolumn[
\icmltitle{Learning Neurosymbolic Generative Models via Program Synthesis}




\begin{icmlauthorlist}
\icmlauthor{Halley Young}{a}
\icmlauthor{Osbert Bastani}{a}
\icmlauthor{Mayur Naik}{a}
\end{icmlauthorlist}

\icmlaffiliation{a}{University of Pennsylvania}
%\icmlaffiliation{goo}{Googol ShallowMind, New London, Michigan, USA}
%\icmlaffiliation{ed}{School of Computation, University of Edenborrow, Edenborrow, United Kingdom}

\icmlcorrespondingauthor{Halley Young}{halleyy@seas.upenn.edu}
%\icmlcorrespondingauthor{Osbert Bastani}{obastani@seas.upenn.edu}

%\icmlkeywords{Machine Learning, ICML}

\vskip 0.3in
]



\printAffiliationsAndNotice{}  % leave blank if no need to mention equal contribution

\begin{abstract}
Significant strides have been made toward designing better generative models in recent years. Despite this progress, however, state-of-the-art approaches are still largely unable to capture complex global structure in data. For example, images of buildings typically contain spatial patterns such as windows repeating at regular intervals; state-of-the-art generative methods can't easily reproduce these structures. We propose to address this problem by incorporating programs representing global structure into the generative model---e.g., a 2D for-loop may represent a configuration of windows. Furthermore, we propose a framework for learning these models by leveraging program synthesis to generate training data. On both synthetic and real-world data, we demonstrate that our approach is substantially better than the state-of-the-art at both generating and completing images that contain global structure.
\end{abstract}

\input{intro}
\input{pipeline}
\input{rel}
\input{model}
\input{synth}

\input{exp}
\input{conc}

%\section*{Acknowledgements}

\bibliographystyle{icml2018}
\bibliography{progstruc.bib}

\clearpage

\onecolumn

\appendix
\input{appendix}

\end{document}

%% file: macros.tex
\usepackage{graphicx}
\usepackage{amsmath,amssymb}
\usepackage{multirow}

\newcommand{\X}{\mathcal{X}}
\newcommand{\Z}{\mathcal{Z}}
\newcommand{\SK}{\mathcal{S}}
\newcommand{\C}{\mathcal{C}}

\newcommand{\kl}{\;\|\;}

\usepackage{amsthm}

\newtheorem{theorem}{Theorem}[section]

%% file: intro.tex
\begin{figure*}[ht]
\centering
\begin{tabular}{cccc}
  \includegraphics[width=1.1in]{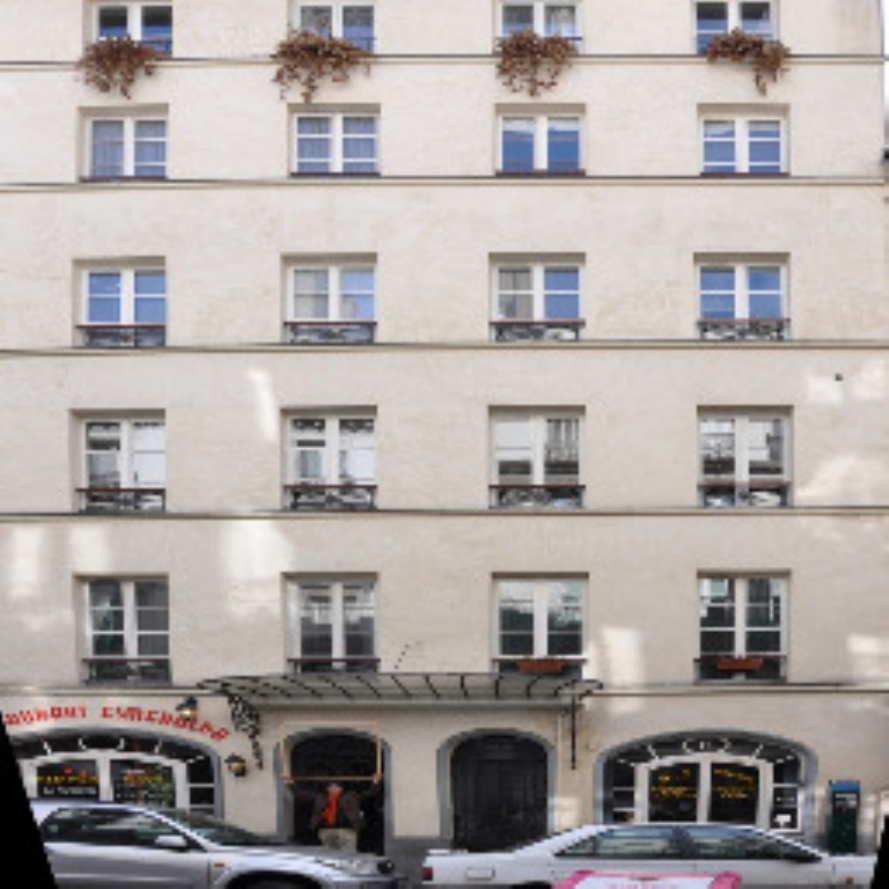} \hspace{0.25cm} & \hspace{0.25cm}
  \includegraphics[width=1.1in]{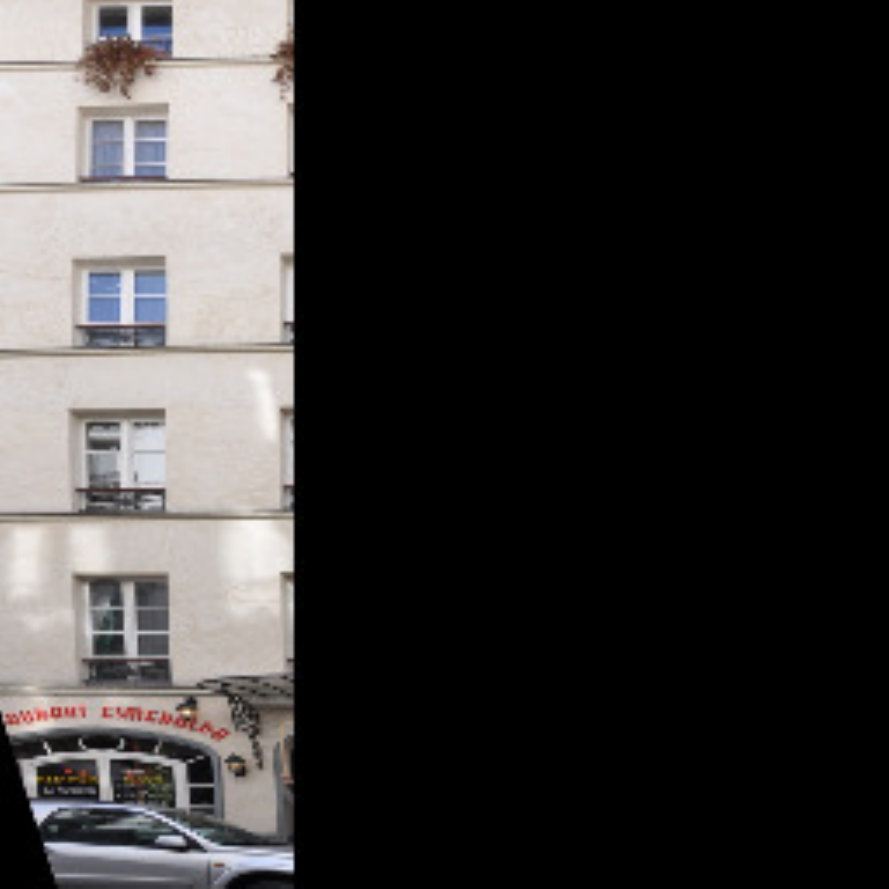} \hspace{0.25cm} & \hspace{0.25cm}
  \includegraphics[width=1.1in]{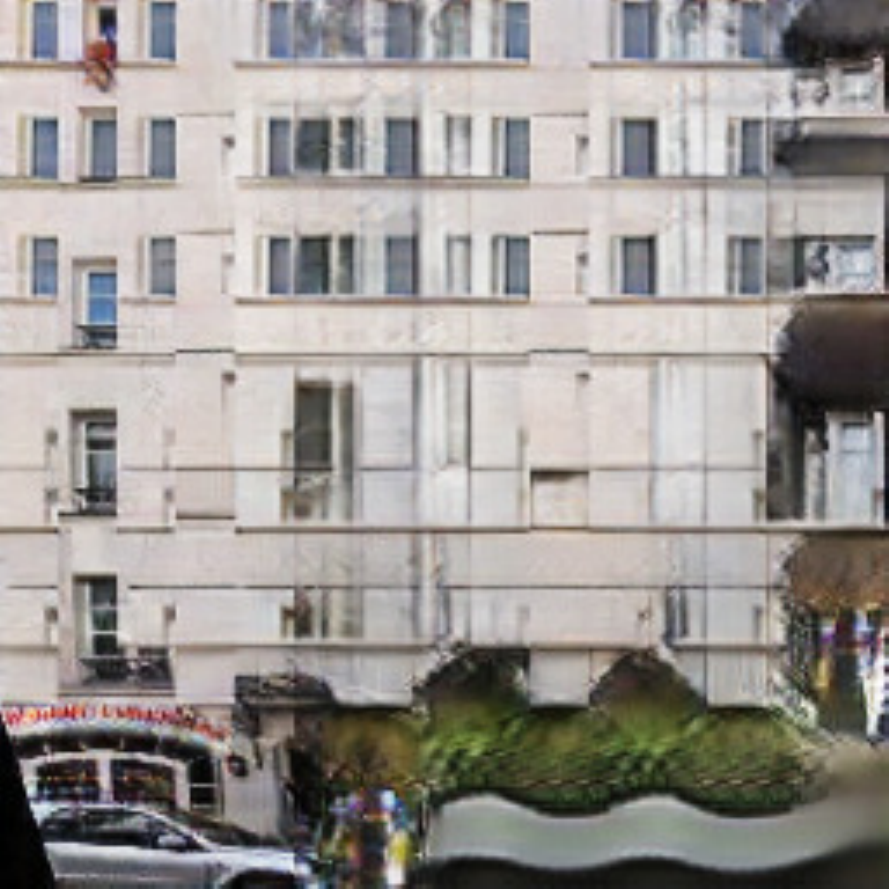} \hspace{0.25cm} & \hspace{0.25cm}
  \includegraphics[width=1.1in]{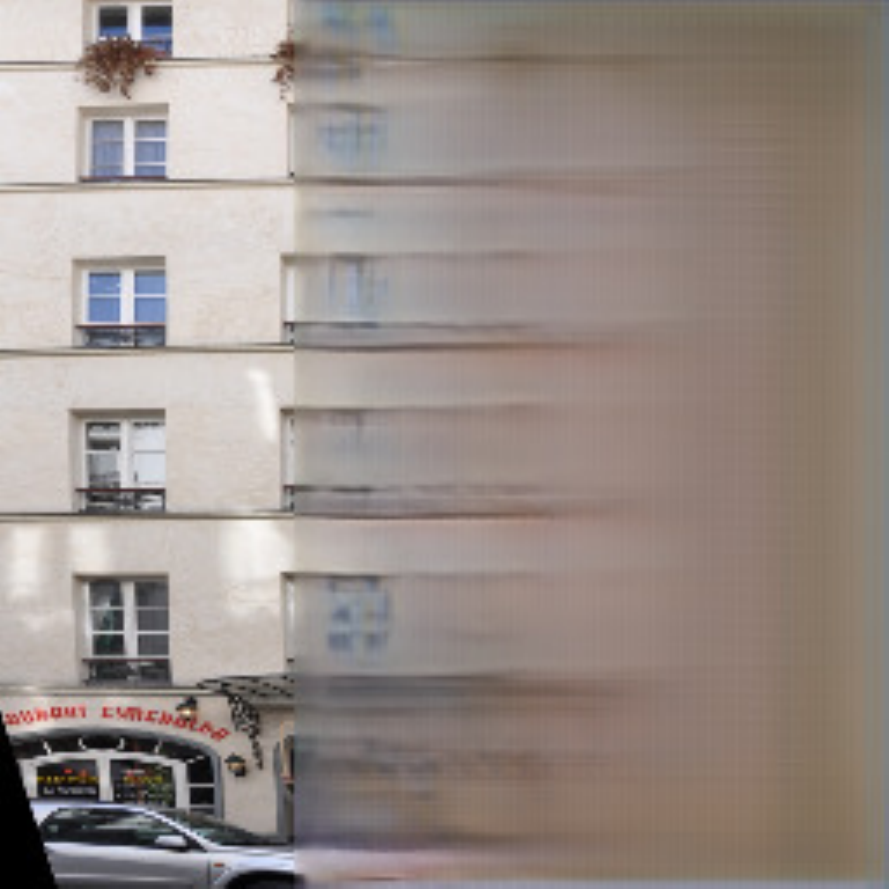} \\
  original image $x^*$ & partial image $x_{\text{part}}$ & completion $\hat{x}$ (ours) & completion $\hat{x}$ (baseline)
\end{tabular}
\caption{The task is to complete the partial image $x_{\text{part}}$ (middle left) into an image that is close to the original image $x^*$ (left). By incorporating programmatic structure into generative models, the completion (middle right) is able to substantially outperform a state-of-the-art baseline~\cite{glcic} (right).}
\label{fig:introcomparison}
\vspace{-0.05in}
\end{figure*}

\section{Introduction}

There has been much interest recently in generative models, following the introduction of both variational autoencoders (VAEs)~\cite{kingma2013auto} and generative adversarial networks (GANs)~\cite{goodfellow2014generative}. These models have successfully been applied to a range of tasks, including image generation~\cite{radford2015unsupervised}, image completion~\cite{glcic}, texture synthesis~\cite{jetchev2017texture,xian2018texturegan}, sketch generation~\cite{ha2017neural}, and music generation~\cite{dieleman2018challenge}.

Despite their successes, generative models still have difficulty capturing global structure. For example, consider the image completion task in Figure~\ref{fig:introcomparison}. The original image (left) is of a building, for which the global structure is a 2D repeating pattern of windows. Given a partial image (middle left), the goal is to predict the completion of the image. As can be seen, a state-of-the-art image completion algorithm has trouble reconstructing the original image (right)~\cite{glcic}.
\footnote{The baseline model performs particularly poorly since our dataset is small. We show on synthetic data that our approach is significantly better even when a large amount of data is available.}
Real-world data often contains such global structure, including repetitions, reflectional or rotational symmetry, or even more complex patterns.

In the past few years, \emph{program synthesis}~\cite{solar2006combinatorial} has emerged as a promising approach to capturing patterns in data~\cite{ellis2015unsupervised,ellis2018learning,valkov2018houdini}. The idea is that simple programs can capture global structure that evades state-of-the-art deep neural networks. A key benefit of using program synthesis is that we can design the space of programs to capture different kinds of structure---e.g., repeating patterns~\cite{ellis2018learning}, symmetries, or spatial structure~\cite{deng2018probabilistic}---depending on the application domain. The challenge is that for the most part, existing approaches have synthesized programs that operate directly over raw data. Since programs have difficulty operating over perceptual data, existing approaches have largely been limited to very simple data---e.g., detecting 2D repeating patterns of simple shapes~\cite{ellis2018learning}.

We propose to address these shortcomings by synthesizing programs that represent the underlying structure of high-dimensional data. In particular, we decompose programs into two parts: (i) a \emph{sketch} $s\in S$ that represents the skeletal structure of the program~\cite{solar2006combinatorial}, with \emph{holes} that are left unimplemented, and (ii) \emph{components} $c\in C$ that can be used to fill these holes. We consider \emph{perceptual components}---i.e., holes in the sketch are filled with raw perceptual data. For example, the program
\begin{align*}
\includegraphics[width=2in]{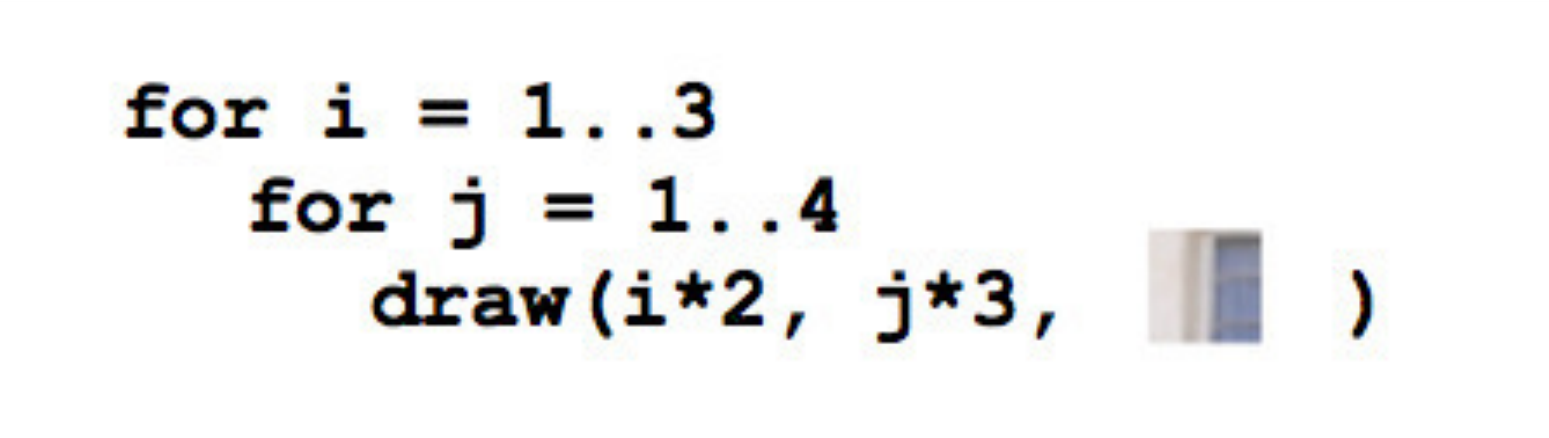}
\end{align*}
represents the structure in the original image $x^*$ in Figure~\ref{fig:introcomparison} (left). The black text is the sketch, and the component is a sub-image taken from the given partial image. Then, the \texttt{draw} function renders the given sub-image at the given position. We call a sketch whose holes are filled with perceptual components a \emph{neurosymbolic program}.

Building on these ideas, we propose an approach called \emph{program-synthesis (guided) generative models} (PS-GM) that combines neurosymbolic programs representing global structure with state-of-the-art deep generative models. By incorporating programmatic structure, PS-GM substantially improves the quality of these state-of-the-art models. As can be seen, the completion produced using PS-GM (middle right of Figure~\ref{fig:introcomparison}) substantially outperforms the baseline.

We show that PS-GM can be used for both generation from scratch and for image completion. The generation pipeline is shown in Figure~\ref{fig:generationexample}. At a high level, PS-GM for generation operates in two phases:
{\setlength{\parskip}{0pt}
\begin{itemize}
\setlength{\itemsep}{0pt}
\item First, it generates  a program that represents the global structure in the image to be generated. In particular, it generates a program $P=(s,c)$ representing the latent global structure in the image (left in Figure~\ref{fig:generationexample}), where $s$ is a sketch and $c$ is a perceptual component.
\item Second, our algorithm executes $P$ to obtain a \emph{structure rendering} $x_{\text{struct}}$ representing the program as an image (middle of Figure~\ref{fig:generationexample}). Then, our algorithm uses a deep generative model to complete $x_{\text{struct}}$ into a full image (right of Figure~\ref{fig:generationexample}). The structure in $x_{\text{struct}}$ helps guide the deep generative model towards images that preserve the global structure.
\end{itemize}
The image-completion pipeline (see Figure~\ref{fig:completionexample}) is similar.}

Training these models end-to-end is challenging, since a priori, ground truth global structure is unavailable. Furthermore, representative global structure is very sparse, so approaches such as reinforcement learning do not scale. Instead, we leverage domain-specific program synthesis algorithms to produce examples of programs that represent global structure of the training data. In particular, we propose a synthesis algorithm tailored to the image domain, which extracts programs with nested for-loops that can represent multiple 2D repeating patterns in images. Then, we use these example programs as supervised training data.

Our programs can capture rich spatial structure in the training data. For example, in Figure~\ref{fig:generationexample}, the program structure encodes a repeating structure of 0's and 2's on the whole image, and a separate repeating structure of 3's on the right-hand side of the image. Furthermore, in Figure~\ref{fig:introcomparison}, the generated image captures the idea that the repeating pattern of windows does not extend to the bottom portion of the image.

\vspace{-0.1in}
\paragraph{Contributions.}

We propose an architecture of generative models that incorporates programmatic structure, as well as an algorithm for training these models (Section~\ref{sec:model}). Our learning algorithm depends on a domain-specific program synthesis algorithm for extracting global structure from the training data; we propose such an algorithm for the image domain (Section~\ref{sec:synth}). Finally, we evaluate our approach on synthetic data and on a real-world dataset of building facades~\cite{facadesdataset}, both on the task of generation from scratch and on generation from a partial image. We show that our approach substantially outperforms several state-of-the-art deep generative models (Section~\ref{sec:exp}).

%% file: pipeline.tex
\begin{figure*}
\centering
\begin{tabular}{ccc}
  \includegraphics[width=2.5in]{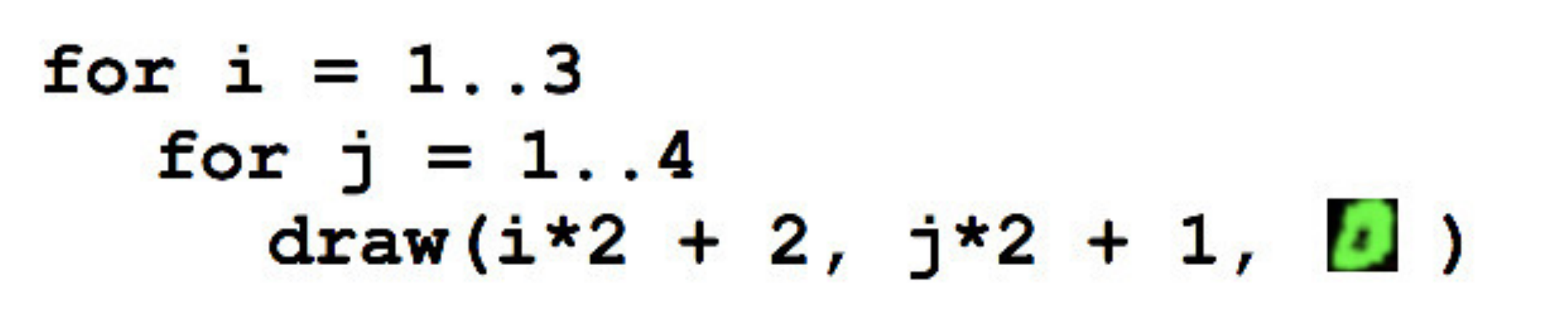} \hspace{0.25cm} & \hspace{0.25cm}
  \includegraphics[width=1.4in]{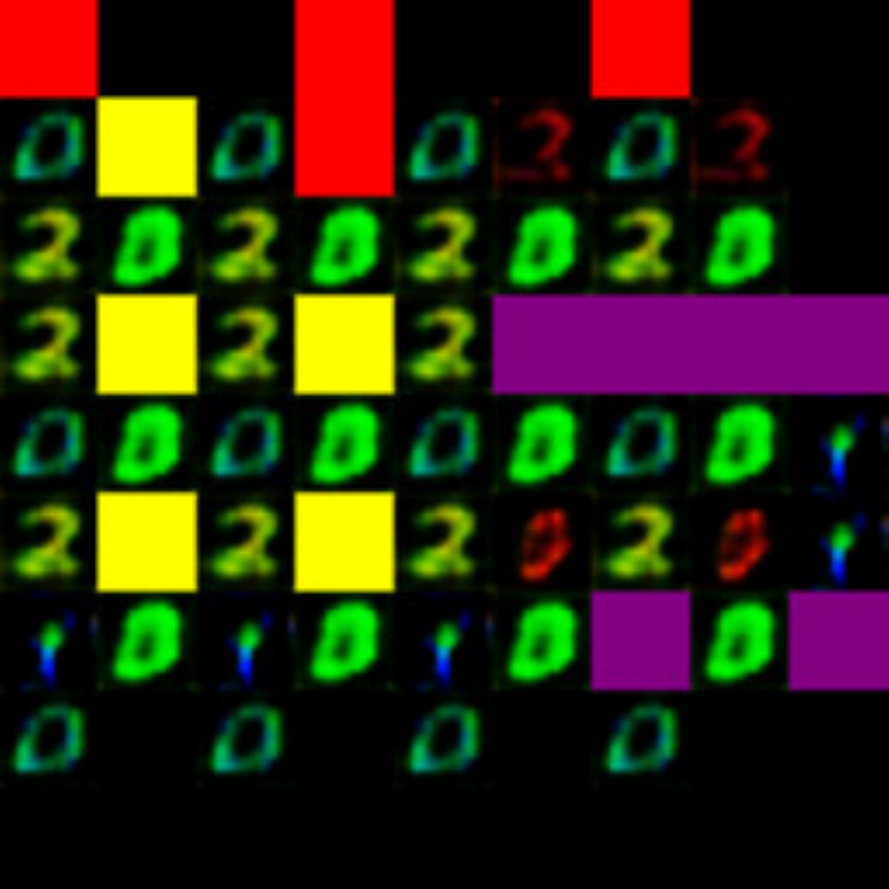} \hspace{0.25cm} & \hspace{0.25cm}
  \includegraphics[width=1.4in]{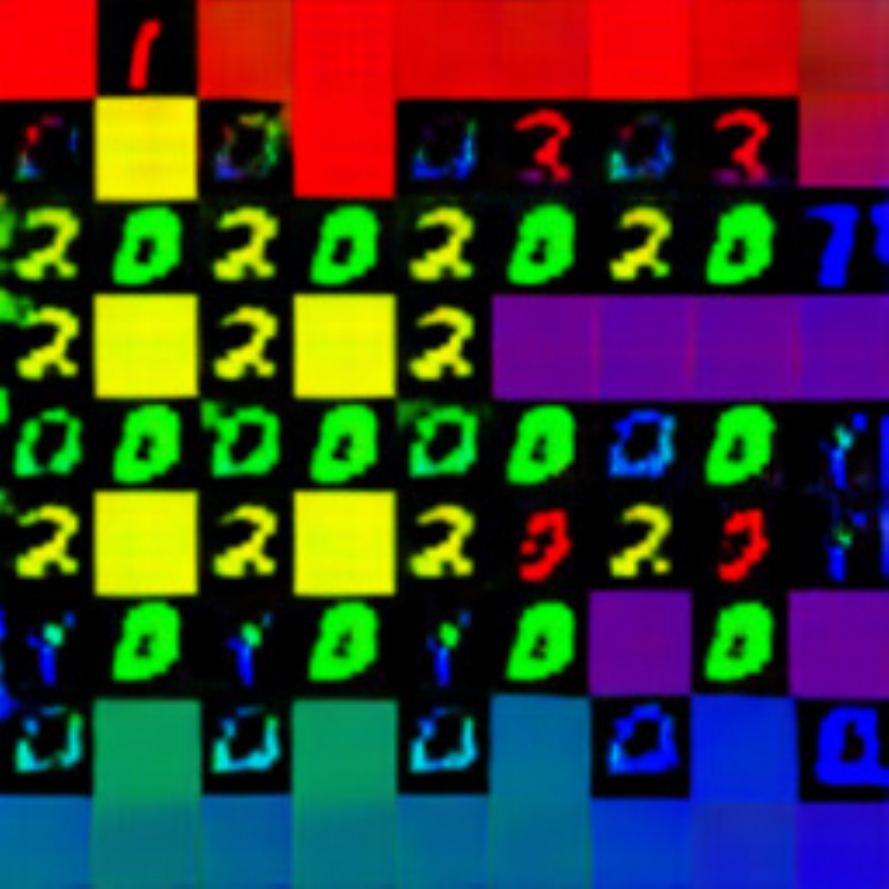} \\
  for loop from sampled program $P$ & structure rendering $x_{\text{struct}}$ & completed image $x$
\end{tabular}
\caption{Our image generation pipeline consists of the following steps: (i) Our generative model samples a latent vector $z\sim p(z)$, and samples a program $P=(s,c)\sim p_{\phi}(s,c\mid z)$ (left). (ii) Our model executes $P$ to obtain a rendering of the program structure $x_{\text{struct}}$ (middle). (iii) Our model samples a completion $x\sim p_{\theta}(x\mid s,c)$ of $x_{\text{struct}}$ into a full image (right).}
\label{fig:generationexample}
\end{figure*}

\begin{figure*}[ht]
\centering
\begin{tabular}{ccc}
  \hspace{0.3in}\includegraphics[width=1.4in]{images-intro-facades-samplefacadethird.pdf}\hspace{0.3in} \hspace{0.05cm} & \hspace{0.05cm}
  \includegraphics[width=2in]{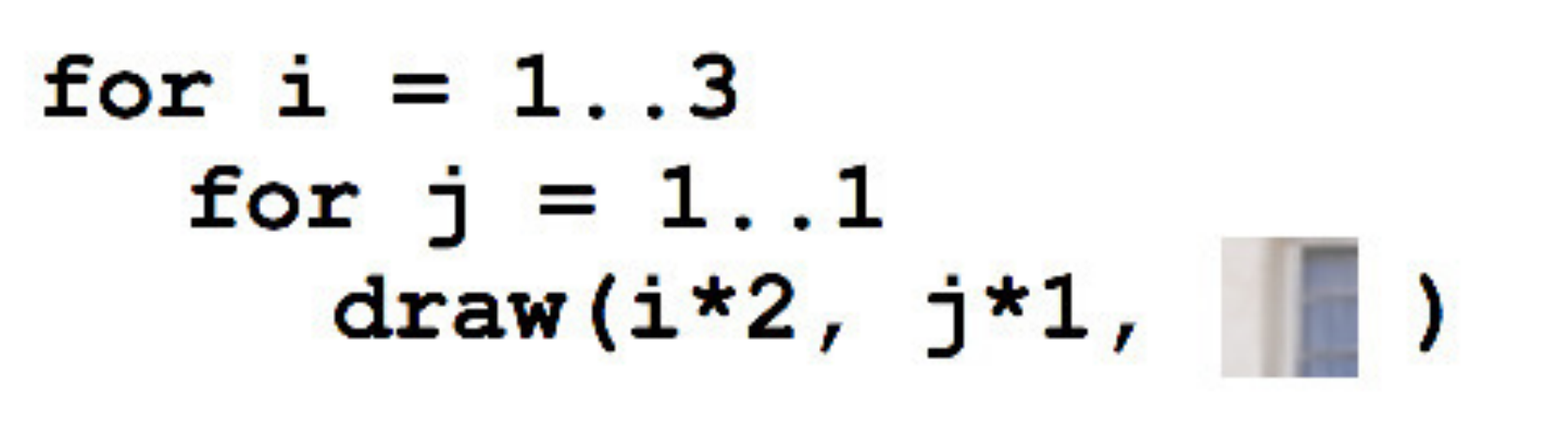} \hspace{0.05cm} & \hspace{0.05cm}
  \includegraphics[width=2in]{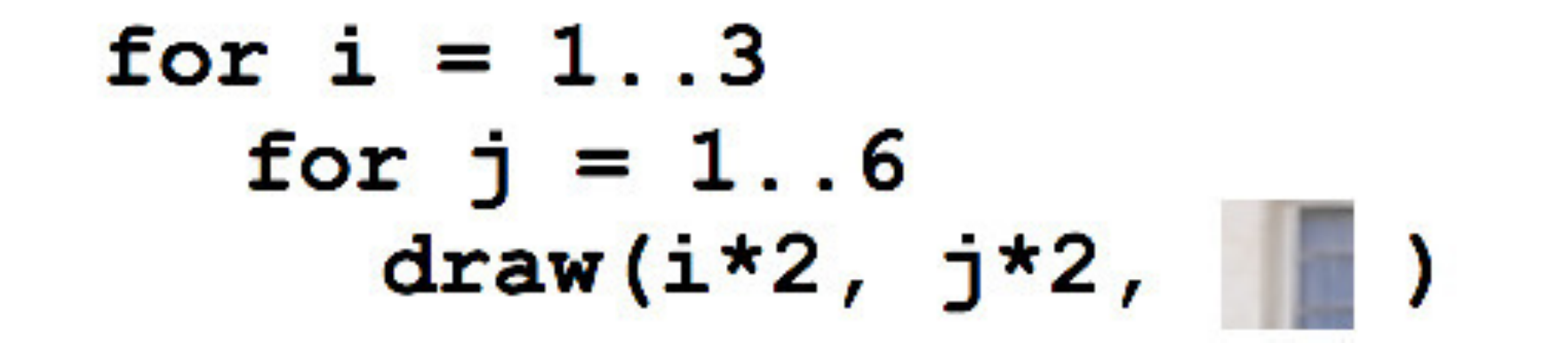} \\
  partial image $x_{\text{part}}$ & synthesized program $P_{\text{part}}$ & extrapolated program $\hat{P}$ \\\\
  \includegraphics[width=1.4in]{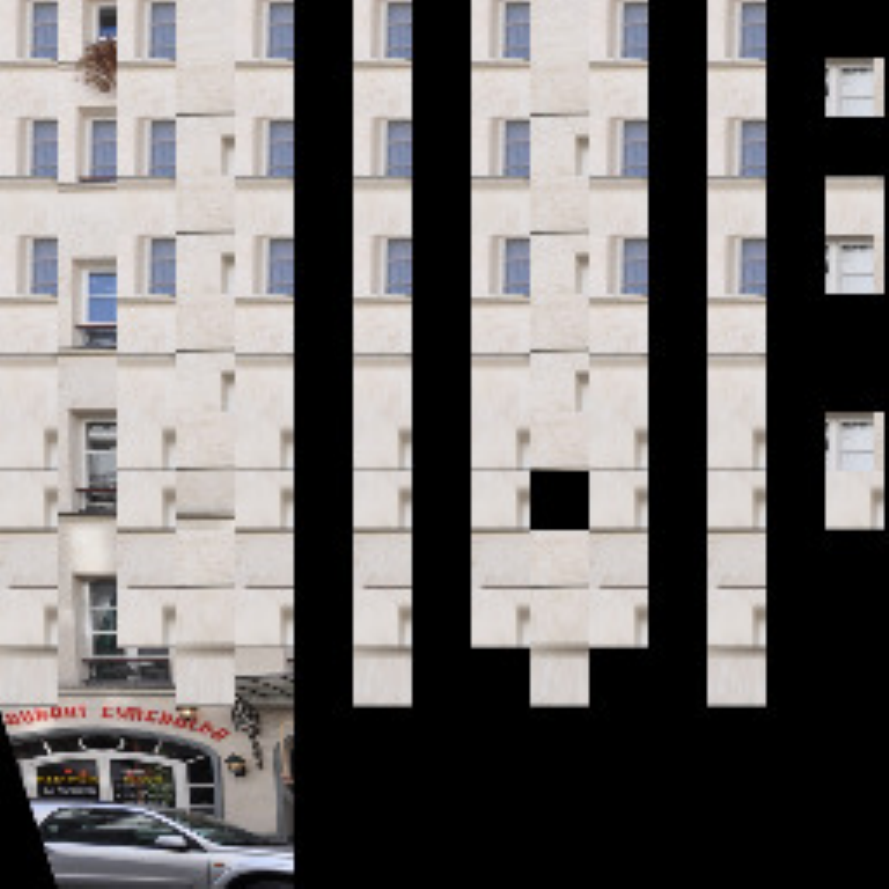} \hspace{0.05cm} & \hspace{0.05cm}
  \includegraphics[width=1.4in]{images-intro-facades-samplethirdextrapcomplete.pdf} \hspace{0.05cm} & \hspace{0.05cm}
  \includegraphics[width=1.4in]{images-intro-facades-samplefacade.pdf} \\
  structure rendering $\hat{x}_{\text{struct+part}}$ & completion $\hat{x}$ (ours) & original image $x^*$
\end{tabular}
\caption{Our image completion pipeline consists of the following steps: (i) Given a partial image $x_{\text{part}}$ (top left), our program synthesis algorithm (Section~\ref{sec:synth}) synthesizes a program $P_{\text{part}}$ representing the structure in the partial image (top middle). (ii) Our model $f$ extrapolates $P_{\text{part}}$ to a program $\hat{P}=f(P_{\text{part}})$ representing the structure of the whole image. (iii) Our model executes $\hat{P}$ to obtain a rendering of the program structure $\hat{x}_{\text{struct+part}}$ (bottom left). (iv) Our model completes $\hat{x}_{\text{struct+part}}$ into an image $\hat{x}$ (bottom middle), which resembles the original image $x^*$ (bottom right).}
\label{fig:completionexample}
\vspace{-0.1in}
\end{figure*}

\vspace{-0.1in}

%% file: rel.tex
\paragraph{Related work.}

There has been growing interest in applying program synthesis to machine learning, for purposes of interpretability~\cite{wang2015falling,verma2018programmatically}, safety~\cite{bastani2018verifiable}, and lifelong learning~\cite{valkov2018houdini}. Most relevantly, there has been interest in using programs to capture structure that deep learning models have difficulty representing~\cite{lake2015human,ellis2015unsupervised,ellis2018learning,pu2018selecting}. For instance, \citet{ellis2015unsupervised} proposes an unsupervised learning algorithm for capturing repeating patterns in simple line drawings; however, not only are their domains simple, but they can only handle a very small amount of noise. Similarly, \citet{ellis2018learning} captures 2D repeating patterns of simple circles and polygons; however, rather than synthesizing programs with perceptual components, they learn a simple mapping from images to symbols as a preprocessing step. The closest work we are aware of is~\citet{valkov2018houdini}, which synthesizes programs with \emph{neural components} (i.e., components implemented as neural networks); however, their application is to lifelong learning, not generation, and to learning with supervision (labels) rather than to unsupervised learning of structure.

Additionally, there has been work extending neural module networks~\cite{andreas2016neural} to generative models~\cite{deng2018probabilistic}. These algorithms essentially learn a collection of neural components that can be composed together based on hierarchical structure. However, they require that the structure be available (albeit in natural language form) both for training the model and for generating new images.

Finally, there has been work incorporating spatial structure into generative models for generating textures~\cite{jetchev2017texture}; however, their work only handles a single infinite repeating 2D pattern. In contrast, we can capture a rich variety of spatial patterns parameterized by a space of programs. For example, the image in Figure~\ref{fig:introcomparison} generated by our technique contains different repeating patterns in different parts of the image.

%% file: model.tex
\section{Generative Models with Latent Structure}
\label{sec:model}

We describe our proposed architecture for generative models that incorporate programmatic structure. For most of this section, we focus on generation; we discuss how we adapt these techniques to image completion at the end. We illustrate our generation pipeline in Figure~\ref{fig:generationexample}.

Let $p_{\theta,\phi}(x)$ be a distribution over a space $\X$ with unknown parameters $\theta,\phi$ that we want to estimate. We study the setting where $x$ is generated based on some latent structure, which consists of a \emph{program sketch} $s\in\SK$ and a \emph{perceptual component} $c\in\C$, and where the structure is in turn generated conditioned on a latent vector $z\in\Z$---i.e.,
\begin{align*}
p_{\theta,\phi}(x)=\int_{\Z}\int_{\C}\sum_{s\in\SK}p_{\theta}(x\mid s,c)p_{\phi}(s,c\mid z)p(z)dcdz.
\end{align*}
Figure~\ref{fig:generationexample} shows an example of a sampled program $P=(s,c)\sim p_{\phi}(s,c\mid z)$ (left), and an example of the sampled completion $x\sim p_{\theta}(x\mid s,c)$ (right). To sample a completion, our model executes $P$ to obtain a structure rendering $x_{\text{struct}}=\text{eval}(P)$ (middle), and then samples a completion based on $x_{\text{struct}}$---i.e., $p_{\theta}(x\mid s,c)=p_{\theta}(x\mid x_{\text{struct}})$.

We now describe our algorithm for learning the parameters $\theta,\phi$ of $p_{\theta,\phi}$, followed by a description of our choices of architecture for $p_{\phi}(s,c\mid z)$ and $p_{\theta}(x\mid s,c)$.

\paragraph{Learning algorithm.}

Given training data $\{x^{(i)}\}_{i=1}^n\subseteq\X$, where $x^{(i)}\sim p_{\theta,\phi}(x)$, the maximum likelihood estimate is
\begin{align*}
\theta_{\text{MLE}}^*,\phi_{\text{MLE}}^*=\operatorname*{\arg\max}_{\theta,\phi}\sum_{i=1}^n\log p_{\theta,\phi}(x^{(i)}).
\end{align*}
Since $\log p_{\theta,\phi}(x)$ is intractable to optimize, we use an approach based on the variational autoencoder (VAE). In particular, we use a variational distribution
\begin{align*}
q_{\tilde{\phi}}(s,c,z\mid x)=q_{\tilde{\phi}}(z\mid s,c)q(s,c\mid x),
\end{align*}
which has parameters $\tilde{\phi}$. Then, we optimize $\tilde{\phi}$ while simultaneously optimizing $\theta,\phi$. Using $q_{\tilde{\phi}}(s,c,z\mid x)$, the \emph{evidence lower bound} on the log-likelihood is
\begin{align}
\log p_{\theta,\phi}(x)
\ge~&\mathbb{E}_{q(s,c,z\mid x)}[\log p_{\theta}(x\mid s,c)] \nonumber \\
&-D_{\text{KL}}(q(s,c,z\mid x)\kl p_{\phi}(s,c\mid z)p(z)) \nonumber \\
\label{eqn:variational1}
=~&\mathbb{E}_{q(s,c\mid x)}[\log p_{\theta}(x\mid s,c)] \\
&+\mathbb{E}_{q(s,c\mid x),q_{\tilde{\phi}}(z\mid s,c)}[\log p_{\phi}(s,c\mid z)] \nonumber \\
&-\mathbb{E}_{q(s,c\mid x)}[D_{\text{KL}}(q_{\tilde{\phi}}(z\mid s,c)\kl p(z))] \nonumber \\
&-H(q(s,c\mid x)), \nonumber
\end{align}
where $D_{\text{KL}}$ is the KL divergence and $H$ is information entropy. Thus, we can approximate $\theta^*,\phi^*$ by optimizing the lower bound (\ref{eqn:variational1}) instead of $\log p_{\theta,\phi}(x)$. However, (\ref{eqn:variational1}) remains intractable since we are integrating over all program sketches $s\in\SK$ and perceptual components $c\in\C$. Using sampling to estimate these integrals would be very computationally expensive. Instead, we propose an approach that uses a single point estimate of $s_x\in\SK$ and $c_x\in\C$ for each $x\in\X$, which we describe below.

\paragraph{Synthesizing structure.}

For a given $x\in \X$, we use \emph{program synthesis} to infer a \emph{single} likely choice $s_x\in\SK$ and $c_x\in\C$ of the latent structure. The program synthesis algorithm must be tailored to a specific domain; we propose an algorithm for inferring for-loop structure in images in Section~\ref{sec:synth}. Then, we use these point estimates in place of the integrals over $\SK$ and $\C$---i.e., we assume that
\begin{align*}
q(s,c\mid x)=\delta(s-s_x)\delta(c-c_x),
\end{align*}
where $\delta$ is the Dirac delta function. Plugging into (\ref{eqn:variational1}) gives
\begin{align}
\label{eqn:variational2}
\log p_{\theta,\phi}(x)
\ge~&\log p_{\theta}(x\mid s_x,c_x) \\
&+\mathbb{E}_{q_{\tilde{\phi}}(z\mid s_x,c_x)}[\log p_{\phi}(s_x,c_x\mid z)] \nonumber \\
&-D_{\text{KL}}(q_{\tilde{\phi}}(z\mid s_x,c_x)\kl p(z)). \nonumber
\end{align}
where we have dropped the degenerate terms $\log\delta(s-s_x)$ and $\log\delta(c-c_x)$ (which are constant with respect to the parameters $\theta,\phi,\tilde{\phi}$). As a consequence, (\ref{eqn:variational1}) decomposes into two parts that can be straightforwardly optimized---i.e.,
\begin{align*}
\log p_{\theta,\phi}(x)\ge~&\mathcal{L}(\theta;x)+\mathcal{L}(\phi,\tilde{\phi};x) \\
\mathcal{L}(\theta;x)=~&\log p_{\theta}(x\mid s_x,c_x) \\
\mathcal{L}(\phi,\tilde{\phi};x)=~&\mathbb{E}_{q_{\tilde{\phi}}(z\mid s_x,c_x)}[\log p_{\phi}(s_x,c_x\mid z)] \\
&-D_{\text{KL}}(q_{\tilde{\phi}}(z\mid s_x,c_x)\kl p(z)),
\end{align*}
where we can optimize $\theta$ and $\phi,\tilde{\phi}$ independently:
\begin{align*}
\theta^*&=\operatorname*{\arg\max}_{\theta}\sum_{i=1}^n\mathcal{L}(\theta;x^{(i)}) \\
\phi^*,\tilde{\phi}^*&=\operatorname*{\arg\max}_{\phi,\tilde{\phi}}\sum_{i=1}^n\mathcal{L}(\phi,\tilde{\phi};x^{(i)}).
\end{align*}

\paragraph{Latent structure VAE.}

Note that $\mathcal{L}(\phi,\tilde{\phi};x)$ is exactly equal to the objective of a VAE, where $q_{\tilde{\phi}}(z\mid s,c)$ is the encoder and $p_{\phi}(s,c\mid z)$ is the decoder---i.e., learning the distribution over latent structure is equivalent to learning the parameters of a VAE. The architecture of this VAE depends on the representation of $s$ and $c$. In the case of for-loop structure in images, we use a sequence-to-sequence VAE.

\paragraph{Generating data with structure.}

The term $\mathcal{L}(\theta;x)$ corresponds to learning a probability distribution (conditioned on the latent structure $s$ and $c$)---e.g., we can estimate this distribution using another VAE. As before, the architecture of this VAE depends on the representation of $s$ and $c$. Rather than directly predicting $x$ based on $s$ and $c$, we can leverage the program structure more directly by first executing the program $P=(s,c)$ to obtain its output $x_{\text{struct}}=\text{eval}(P)$, which we call a \emph{structure rendering}. In particular, $x_{\text{struct}}$ is a more direct representation of the global structure represented by $P$, so it is often more suitable to use as input to a neural network. The middle of Figure~\ref{fig:generationexample} shows an example of a structure rendering for the program on the left. Then, we can train a model $p_{\theta}(x\mid s,c)=p_{\theta}(x\mid x_{\text{struct}})$.

In the case of images, we use a VAE with convolutional layers for the encoder $q_{\phi}$ and transpose convolutional layers for the decoder $p_{\theta}$.  Furthermore, instead of estimating the entire distribution $p_{\theta}(x\mid s,c)$, we also consider two non-probabilistic approaches that directly predict $x$ from $x_{\text{struct}}$, which is an image completion problem. We can solve this problem using GLCIC, a state-of-the-art image completion model~\cite{glcic}. We can also use CycleGAN~\cite{cyclegan}, which solves the more general problem of mapping a training set of structured renderings $\{x_{\text{struct}}\}$ to a training set of completed images $\{x\}$. \footnote{Pix2Pix~\cite{isola2017image} may seem more appropriate since it takes training pairs $(x_{\text{struct}},x)$, but CycleGAN outperformed it.}

\paragraph{Image completion.}

In image completion, we are given a set of training pairs $(x_{\text{part}},x^*)$, and the goal is to learn a model that predicts the complete image $x^*$ given a partial image $x_{\text{part}}$. Compared to generation, our likelihood is now conditioned on $x_{\text{part}}$---i.e., $p_{\theta,\phi}(x\mid x_{\text{part}})$. Now, we describe how we modify each of our two models $p_{\theta}(x\mid s,c)$ and $p_{\phi}(s,c\mid z)$ to incorporate this extra information.

First, the programmatic structure is no longer fully latent, since we can observe partial programmatic structure in $x_{\text{part}}$. In particular, we can leverage our program synthesis algorithm to help perform completion. We first synthesize programs $P^*$ and $P_{\text{part}}$ representing the global structure in $x^*$ and $x_{\text{part}}$, respectively. Then, we can train a model $f$ that predicts $P^*$ given $P_{\text{part}}$---i.e., it extrapolates $P_{\text{part}}$ to a program $\hat{P}=f(P_{\text{part}})$ representing the structure of the whole image. Thus, unlike generation, where we sample a program $\hat{P}=(s,c)\sim p_{\phi}(s,c\mid z)$, we use the extrapolated program $\hat{P}=f(P_{\text{part}})$.

The second model $p_{\theta}(x\mid s,c)$ for the most part remains the same, except when we execute $\hat{P}=(s,c)$ to obtain a structure rendering $x_{\text{struct}}$, we render onto the partial image $x_{\text{part}}$ instead of onto a blank image to obtain the final rendering $x_{\text{struct+part}}$. Then, we complete the structure rendering $x_{\text{struct+part}}$ into a prediction of the full image $\hat{x}$ as before (i.e., using a VAE, GLCIC, or CycleGAN).

Our image completion pipeline is shown in Figure~\ref{fig:completionexample}, including the given partial image (top left), the program $P_{\text{part}}$ synthesized from the partial image (top middle), the extrapolated program $\hat{P}$ (top right), the structure rendering $x_{\text{struct+part}}$ (bottom left), and the predicted full image $\hat{x}$ (bottom middle).

%% file: synth.tex
\section{Synthesizing Programmatic Structure}
\label{sec:synth}

\paragraph{Image representation.}

Since the images we work with are very high dimensional, for tractability, we assume that each image $x\in\mathbb{R}^{NM\times NM}$ is divided into a grid containing $N$ rows and $N$ columns, where each grid cell has size $M\times M$ pixels (where $M\in\mathbb{N}$ is a hyperparameter). For example, this grid structure is apparent in Figure~\ref{fig:completionexample} (top right), where $N=15$, $M=17$ and $N = 9$, $M = 16$ for the facade and synthetic datasets respectively. For $t,u\in[N]=\{1,...,N\}$, we let $x_{tu}\in\mathbb{R}^{M\times M}$ denote the sub-image at the $(t,u)$ position in the $N\times N$ grid.

\paragraph{Program grammar.}

Given this structure, we consider programs that draw 2D repeating patterns of $M\times M$ sub-images on the grid. More precisely, we consider programs
\begin{align*}
P=((s_1,c_1),...,(s_k,c_k))\in(S\times C)^k
\end{align*}
that are length $k$ lists of pairs consisting of a sketch $s\in S$ and a perceptual component $c\in C$; here, $k\in\mathbb{N}$ is a hyperparameter.
\footnote{So far, we have assumed that a program is a single pair $P=(s,c)$, but the generalization to a list of pairs is straightforward.}
A sketch $s\in S$ has form
\begin{align*}
s=~
&\textbf{for}~(i,j)\in\{1,...,n\}\times\{1,...,n'\}~\textbf{do} \\
&\hspace{0.2in}~\text{draw}(a\cdot i+b,~a'\cdot j+b',~\textbf{??}) \\
&\textbf{end for}
\end{align*}
where $n,a,b,n',a',b'\in\mathbb{N}$ are undetermined parameters that must satisfy $a\cdot n+b\le N$ and $a'\cdot n'+b'\le N$, and where $\textbf{??}$ is a hole to be filled by a perceptual component, which is an $M\times M$ sub-image $c\in\mathbb{R}^{M\times M}$.
\footnote{For colored images, we have $I\in\mathbb{R}^{M\times M\times3}$.}
Then, upon executing the $(i,j)$ iteration of the for-loop, the program renders sub-image $I$ at position $(t,u)=(a\cdot i+b,a'\cdot j+b')$ in the $N\times N$ grid. Figure~\ref{fig:completionexample} (top middle) shows an example of a sketch $s$ where its hole is filled with a sub-image $c$, and Figure~\ref{fig:completionexample} (bottom left) shows the image rendered upon executing $P=(s,c)$. Figure~\ref{fig:generationexample} shows another such example.

\paragraph{Program synthesis problem.}

Given a training image $x\in\mathbb{R}^{NM\times NM}$, our program synthesis algorithm outputs the parameters $n_h,a_h,b_h,n_h',a_h',b_h'$ of each sketch $s_h$ in the program (for $h\in[k]$), along with a perceptual component $c_h$ to fill the hole in sketch $s_h$. Together, these parameters define a program $P=((s_1,c_1),...,(s_k,c_k))$.

The goal is to synthesize a program that faithfully represents the global structure in $x$. We capture this structure using a boolean tensor $B^{(x)}\in\{0,1\}^{N\times N\times N\times N}$, where
\begin{align*}
B^{(x)}_{t,u,t',u'}=
\begin{cases}
1&\text{if}~d(x_{tu},x_{t'u'})\le\epsilon \\
0&\text{otherwise},
\end{cases}
\end{align*}
where $\epsilon\in\mathbb{R}_+$ is a hyperparameter, and $d(I,I')$ is a distance metric between on the space of sub-images. In our implementation, we use a weighted sum of earthmover's distance between the color histograms of $I$ and $I'$, and the number of SIFT correspondences between $I$ and $I'$.

Additionally, we associate a boolean tensor with a given program $P=((s_1,c_1),...,(s_k,c_k))$. First, for a sketch $s\in S$ with parameters $a,b,n,a',b',n'$, we define
\begin{align*}
\text{cover}(s)=\{(a\cdot i+b,a'\cdot j+b')\mid i\in[n],j\in[n']\},
\end{align*}
i.e., the set of grid cells where sketch renders a sub-image upon execution. Then, we have
\begin{align*}
B_{t,u,t',u'}^{(s)}=
\begin{cases}
1&\text{if}~(t,u),(t',u')\in\text{cover}(s) \\
0&\text{otherwise},
\end{cases}
\end{align*}
i.e., $B_{t,u,t',u'}^{(s)}$ indicates whether the sketch $s$ renders a sub-image at both of the grid cells $(t,u)$ and $(t',u')$. Then,
\begin{align*}
B^{(P)}=B^{(s_1)}\vee...\vee B^{(s_k)},
\end{align*}
where the disjunction of boolean tensors is defined elementwise. Intuitively, $B^{(P)}$ identifies the set of pairs of grid cells $(t,u)$ and $(t',u')$ that are equal in the image rendered upon executing each pair $(s,c)$ in $P$.
\footnote{Note that the covers of different sketches in $P$ can overlap; we find that ignoring this overlap does not significantly impact our results.}

Finally, our program synthesis algorithm aims to solve the following optimization problem:
\begin{align}
\label{eqn:synthprob}
P^*&=\operatorname*{\arg\max}_P\ell(P;x) \\
\ell(P;x)&=\|B^{(x)}\wedge B^{(P)}\|_1+\lambda\|\neg B^{(x)}\wedge\neg B^{(P)}\|_1, \nonumber
\end{align}
where $\wedge$ and $\neg$ are applied elementwise, and $\lambda\in\mathbb{R}_+$ is a hyperparameter. In other words, the objective of (\ref{eqn:synthprob}) is the number of true positives (i.e., entries where $B^{(P)}=B^{(x)}=1$), and the number of false negatives (i.e., entries where $B^{(P)}=B^{(x)}=0$), and computes their weighted sum. Thus, the objective of (\ref{eqn:synthprob}) measures for how well $P$ represents the global structure of $x$.

For tractability, we restrict the search space in (\ref{eqn:synthprob}) to programs of the form
\begin{align*}
P&=((s_1,c_1),...,(s_k,c_k))\in(S\times\hat{C})^k \\
\hat{C}&=\{x_{tu}\mid t,u\in[N]\}.
\end{align*}
In other words, rather than searching over all possible sub-images $c\in\mathbb{R}^{M\times M}$, we only search over the sub-images that actually occur in the training image $x$.  This may lead to a slightly sub-optimal solution, for example, in cases where the optimal sub-image to be rendered is in fact an interpolation between two similar but distinct sub-images in the training image.  However, we found that in practice this simplifying assumption still produced viable results.

\begin{algorithm}[t]
\begin{algorithmic}
\STATE {\bf Input:} $X=\{x\}\subseteq\mathbb{R}^{NM\times NM}$
\STATE $\hat{C}\gets\{x_{tu}\mid t,u\in[N]\}$
\STATE $P\gets\varnothing$
\FOR{$h\in\{1,...,k\}$}
\STATE $s_h,c_h=\operatorname*{\arg\max}_{(s,c)\in S\times\hat{C}}\ell(P_{h-1}\cup\{(s,c)\};x)$
\STATE $P\gets P\cup\{(s_h,c_h)\}$
\ENDFOR
\STATE {\bf Output:} $P$
\end{algorithmic}
\caption{Synthesizes a program $P$ representing the global structure of a given image $x\in\mathbb{R}^{NM\times NM}$.}
\label{alg:synth}
\end{algorithm}

\paragraph{Program synthesis algorithm.}

Exactly optimizing (\ref{eqn:synthprob}) is in general an NP-complete problem. Thus, our program synthesis algorithm uses a partially greedy heuristic. In particular, we initialize the program to $P=\varnothing$. Then, on each iteration, we enumerate all pairs $(s,c)\in S\times\hat{C}$ and determine the pair $(s_h,c_h)$ that most increases the objective in (\ref{eqn:synthprob}), where $\hat{C}$ is the set of all sub-images $x_{tu}$ for $t,u\in[N]$. Finally, we add $(s_h,c_h)$ to $P$. We show the full algorithm in Algorithm~\ref{alg:synth}. We have the following straightforward guarantee:
\begin{theorem}
If $\lambda=0$, then $\ell(\hat{P};x)\ge(1-e^{-1})\ell(P^*;x)$, where $\hat{P}$ is returned by Algorithm~\ref{alg:synth} and $P^*$ solves (\ref{eqn:synthprob}).
\end{theorem}
\begin{proof}
If $\lambda=0$, then optimizing $\ell(P;x)$ is equivalent to set cover, where the items are tuples
\begin{align*}
\{(t,u,t',u')\in[N]^4\mid B_{t,u,t',u'}^{(x)}=1\},
\end{align*}
and the sets are $(s,c)\in S\times\hat{C}$. The theorem follows from~\cite{hochbaum1997approximating}.
\end{proof}
In general, (\ref{eqn:synthprob}) is not submodular, but we find that the greedy heuristic still works well in practice.

%% file: exp.tex
\section{Experiments}
\label{sec:exp}

We perform two experiments---one for generation from scratch and one for image completion. We find substantial improvement in both tasks. Details on neural network architectures are in Appendix~\ref{sec:expdetailsappendix}, and additional examples for image completion are in Appendix~\ref{sec:expappendix}.

\subsection{Datasets}

\paragraph{Synthetic dataset.}

We developed a synthetic dataset based on MNIST. Each image consists of a $9\times9$ grid, where each grid cell is $16\times16$ pixels. Each grid cell is either filled with a colored MNIST digit or a solid color background. The program structure is a 2D repeating pattern of an MNIST digit; to add natural noise, we each iteration of the for-loop in a sketch $s_h$ renders different MNIST digits, but with the same MNIST label and color. Additionally, we chose the program structure to contain correlations characteristic of real-world images---e.g., correlations between different parts of the program, correlations between the program and the background, and noise in renderings of the same component. Examples are shown in Figure~\ref{fig:genexpimages}. We give details of how we constructed this dataset in Appendix~\ref{sec:expdetailsappendix}. This dataset contains 10,000 training and 500 test images.

\paragraph{Facades dataset.}

Our second dataset consists of 1855 images (1755 training, 100 testing) of building facades.\footnote{We chose a large training set since our dataset is so small.}  These images were all scaled to a size of $256\times256\times3$ pixels, and were divided into a grid of $15\times15$ cells each of size 17 or 18 pixels.  These images contain repeating patterns of objects such as windows and doors.

\subsection{Generation from Scratch}
\label{sec:expgen}

\paragraph{Experimental setup.}

We evaluate our approach PS-GM for generation from scratch on the synthetic dataset---the facades dataset was too small to produce meaningful results. As described in Section~\ref{sec:model}, we use Algorithm~\ref{alg:synth} to synthesize a program $P_x=(s_x,c_x)$ representing each training image $x\in X_{\text{train}}$. Then, we train the encoder $q_{\tilde{\phi}}(z\mid s,c)$ and the decoder $p_{\phi}(s,c\mid z)$ on the training set $\{P_x\mid x\in X\}$.

For the second stage of PS-GM (i.e., completing the structure rendering $x_{\text{struct}}$ into an image $x$), we use a variational encoder-decoder (VED)
\begin{align*}
p_{\theta}(x\mid s,c)=\int p_{\theta}(x\mid w)\cdot q_{\theta}(w\mid x_{\text{struct}})dw,
\end{align*}
where $q_{\theta}(w\mid x_{\text{struct}})$ encodes a structure rendering $x_{\text{struct}}$ into a latent vector $w$, and $p_{\theta}(x\mid w)$ decodes $w$ into a complete image $x$. We train $p_{\theta}$ and $q_{\theta}$ using the VAE training loss, except we minimize the distance between a structure rendering $x_{\text{struct}}$ and the original image $x^*$. Additionally, we trained a CycleGAN model to map structure renderings to complete images, by giving the CycleGAN unaligned pairs of $x_{\text{struct}}$ and $x$ as training data. We compare our VED model to a VAE~\cite{kingma2013auto}, and our CycleGAN model to SpatialGAN~\cite{spatialgan}.  

\paragraph{Results.}

\begin{table}
\centering
\begin{tabular}{lr}
  \toprule
  \textbf{Model} & \multicolumn{1}{c}{\textbf{Score}} \\
  \midrule
  PS-GM (CycleGAN) & {\bf 85.51} \\
  BL (SpatialGAN) & 258.68 \\
  PS-GM (VED) & {\bf 59414.7} \\
  BL (VAE) & 60368.4 \\
  \midrule
  PS-GM (VED Stage 1 $p_{\phi}(s,c\mid z)$) & 32.0 \\
  PS-GM (VED Stage 2 $p_{\theta}(x\mid s,c)$) & 59382.6 \\
  \bottomrule
\end{tabular}
\caption{Performance of our approach PS-GM versus the baseline (BL) for generation from scratch. We report Fr\'echet inception distance for GAN-based models, and negative log-likelihood for the VAE-based models}
\label{tab:genexpresults}
\end{table}

\begin{figure}[ht]
\centering
\begin{tabular}{ccc}
  \includegraphics[width=0.12\textwidth]{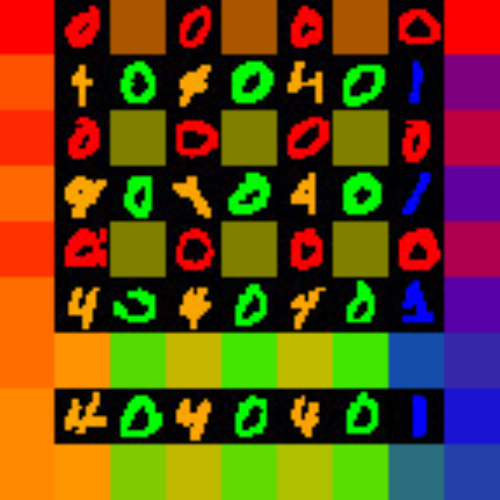} &
  \includegraphics[width=0.12\textwidth]{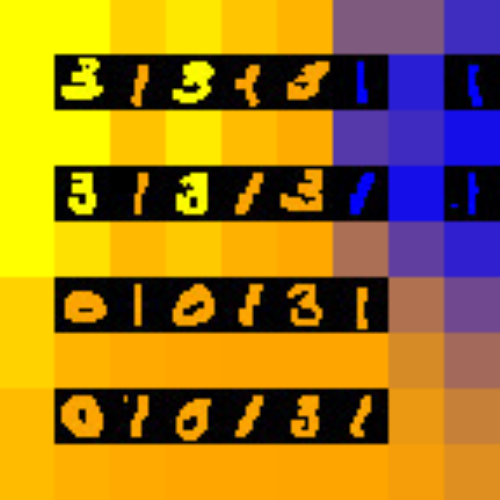} &
  \includegraphics[width=0.12\textwidth]{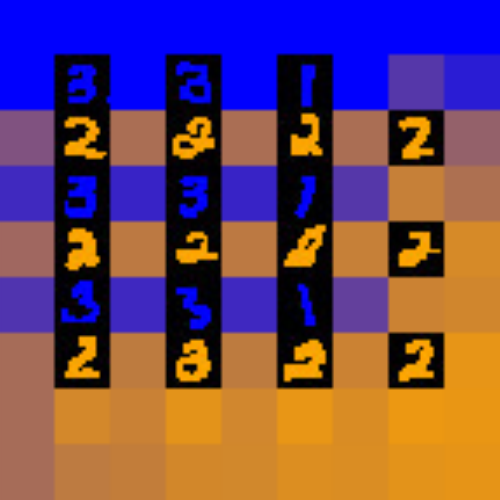} \\
  \multicolumn{3}{c}{Original Images} \vspace{5pt} \\
  \includegraphics[width=0.12\textwidth]{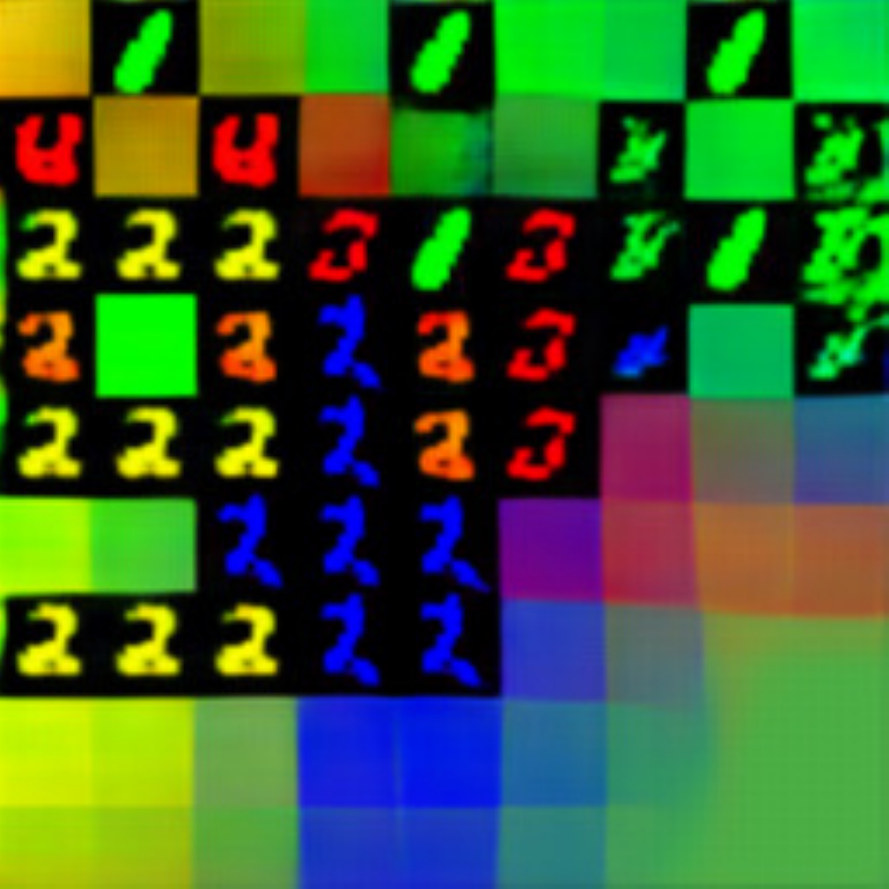} &
  \includegraphics[width=0.12\textwidth]{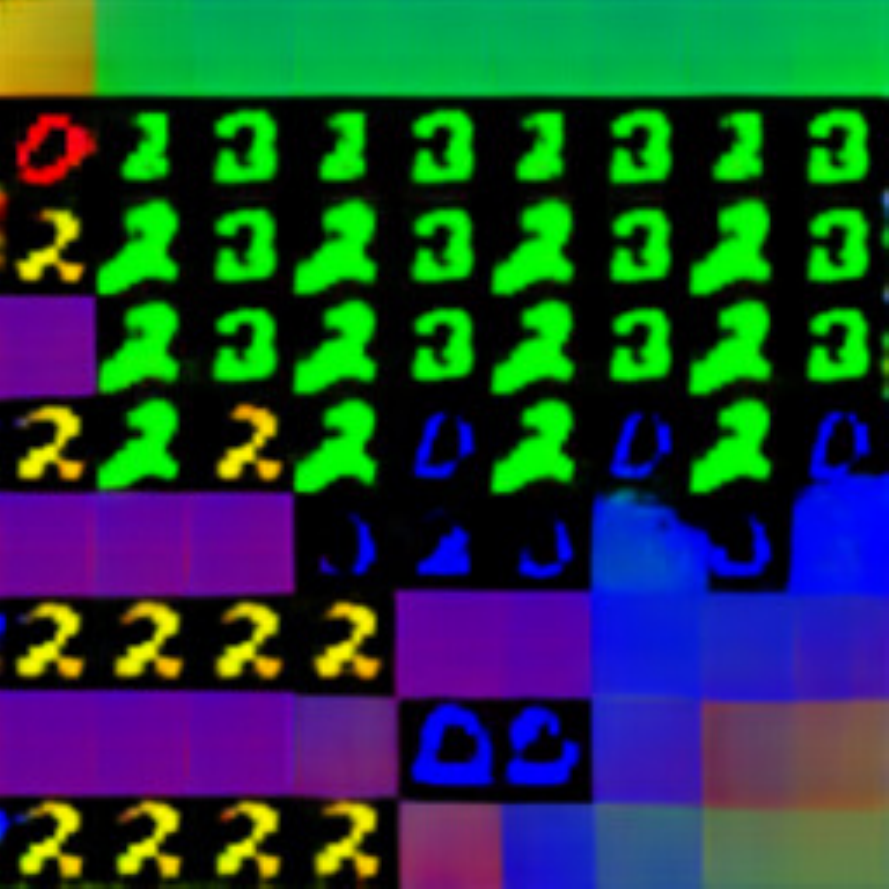} &
  \includegraphics[width=0.12\textwidth]{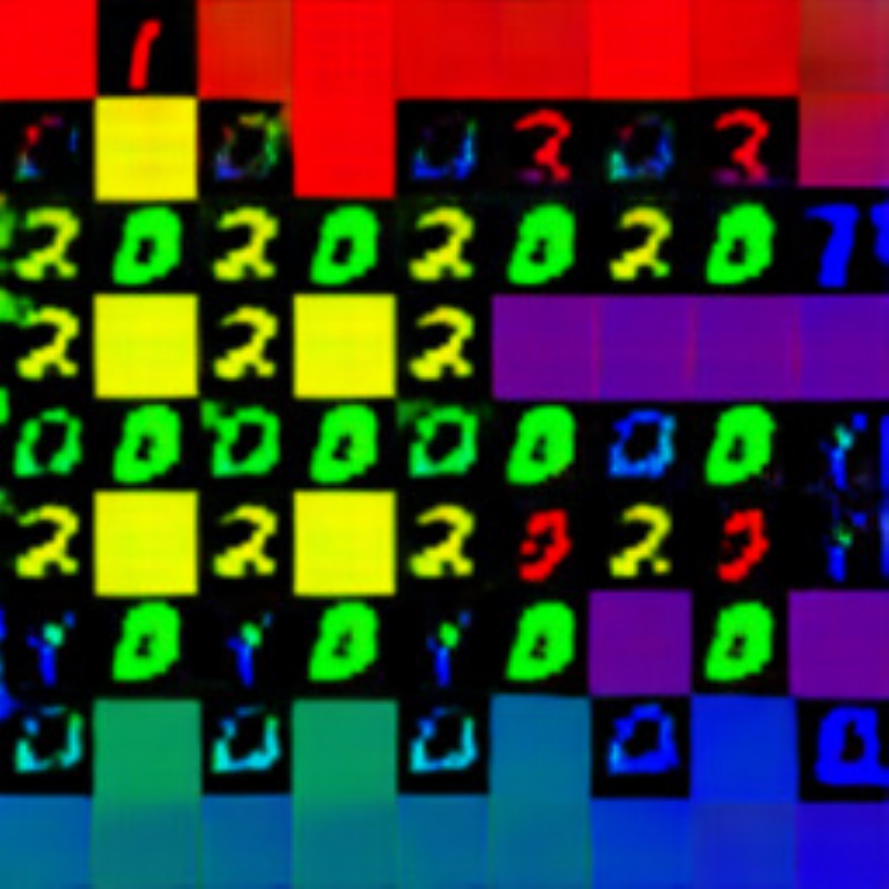} \\
  \multicolumn{3}{c}{PS-GM (CycleGAN)} \vspace{5pt} \\
  \includegraphics[width=0.12\textwidth]{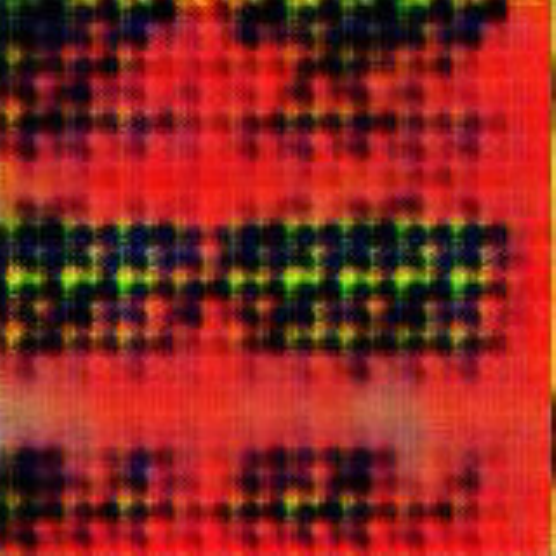} &
  \includegraphics[width=0.12\textwidth]{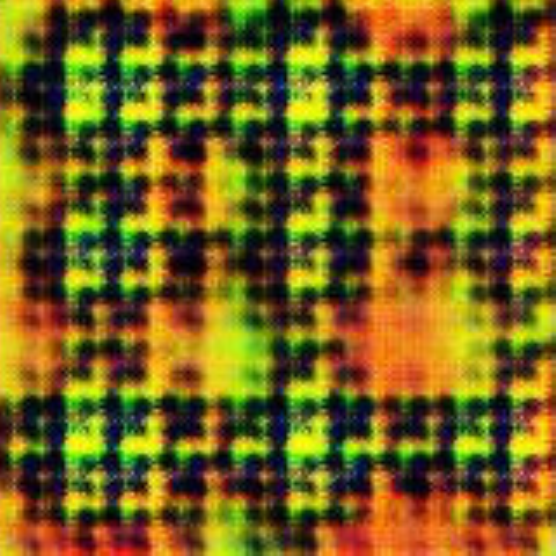} &
  \includegraphics[width=0.12\textwidth]{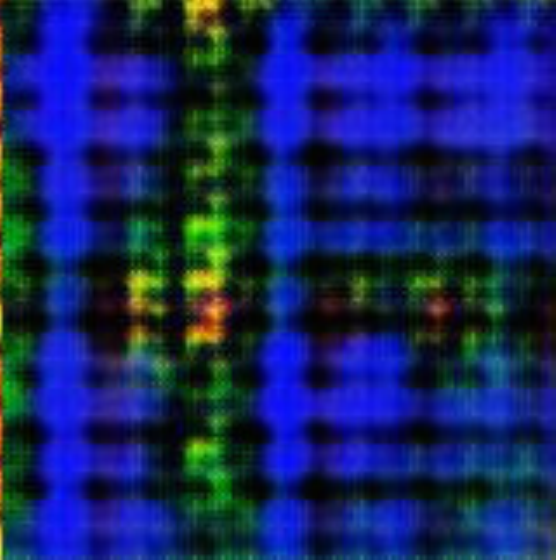} \\
  \multicolumn{3}{c}{Baseline (SpatialGAN)} \vspace{5pt} \\
  \includegraphics[width=0.12\textwidth]{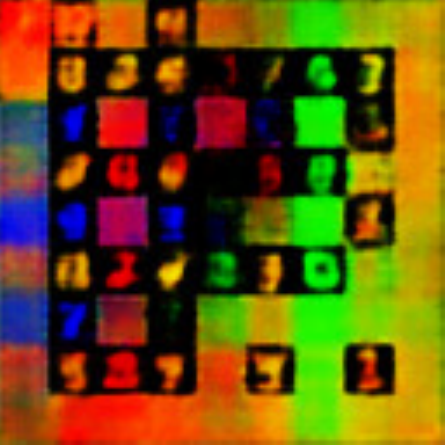} &
  \includegraphics[width=0.12\textwidth]{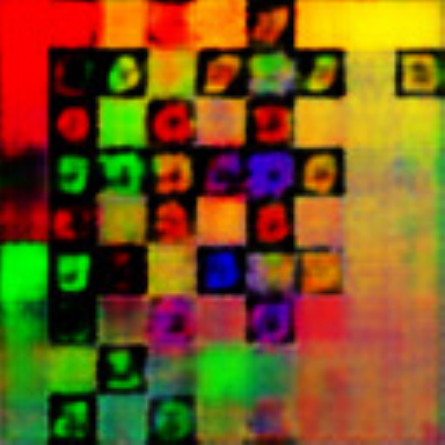} &
  \includegraphics[width=0.12\textwidth]{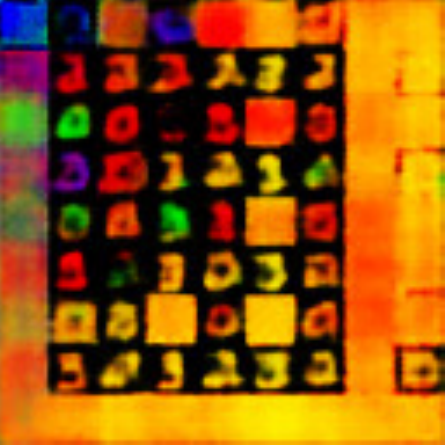} \\
  \multicolumn{3}{c}{PS-GM (VED)} \vspace{5pt} \\
  \includegraphics[width=0.12\textwidth]{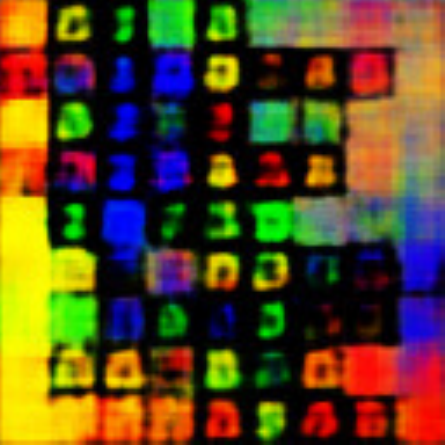} &
  \includegraphics[width=0.12\textwidth]{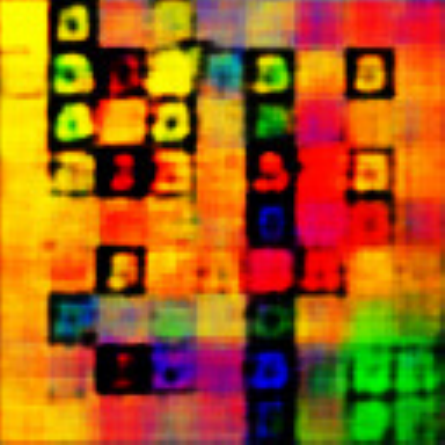} &
  \includegraphics[width=0.12\textwidth]{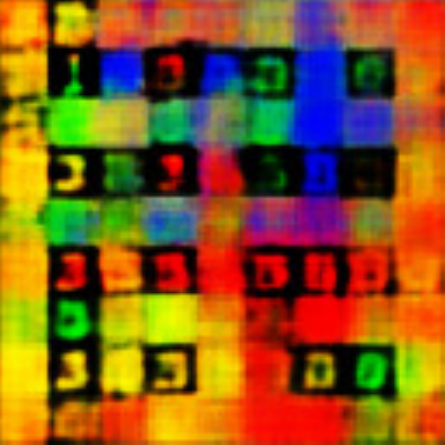} \\
  \multicolumn{3}{c}{Baseline (VAE)}
\end{tabular} 
\caption{Examples of synthetic images generated using our approach, PS-GM (with VED and CycleGan), and the baseline (a VAE and a SpatialGAN). Images in different rows are unrelated since the task is generation from scratch.}
\label{fig:genexpimages}
\end{figure}

\begin{figure*}
\centering
\begin{tabular}{ccccccc}
\includegraphics[width=0.12\textwidth]{images-synth-full1.pdf} &
\includegraphics[width=0.12\textwidth]{images-synth-full2.pdf} &
\includegraphics[width=0.12\textwidth]{images-synth-full3.pdf} & \hspace{0.1in} &
\includegraphics[width=0.12\textwidth]{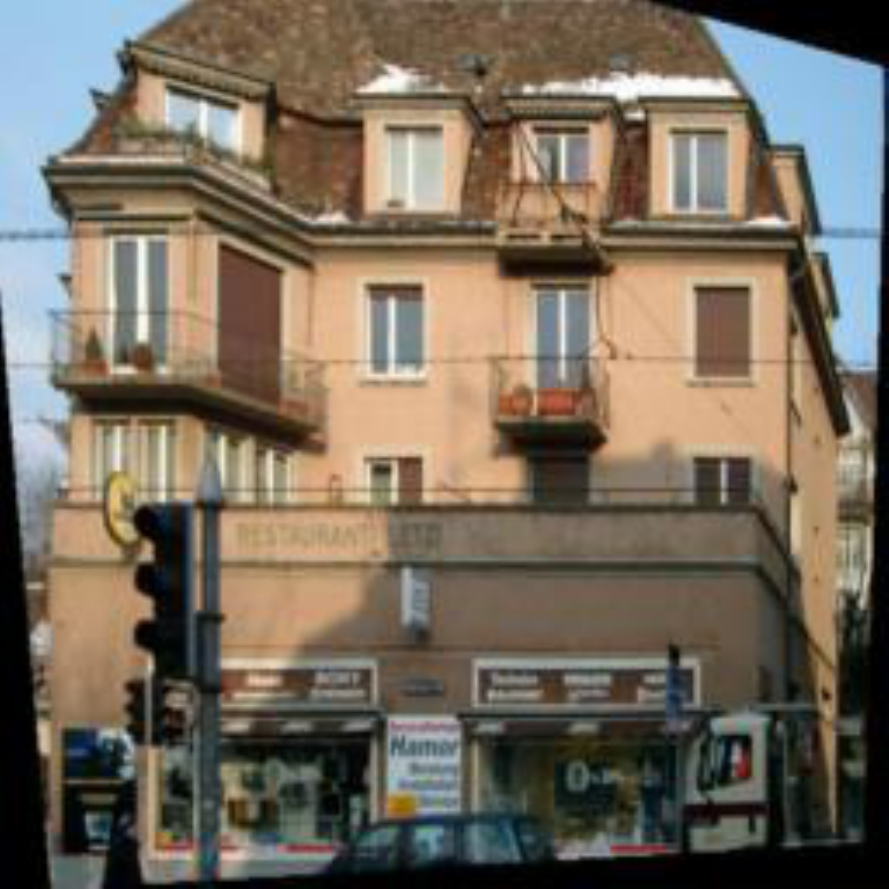} &
\includegraphics[width=0.12\textwidth]{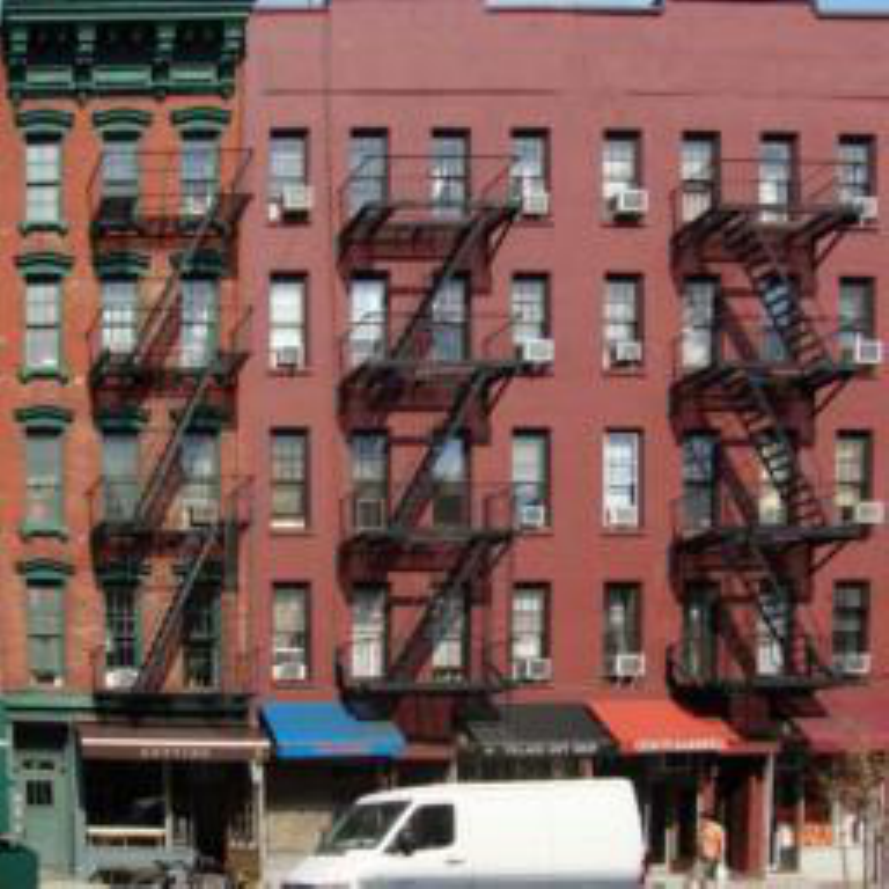} &
\includegraphics[width=0.12\textwidth]{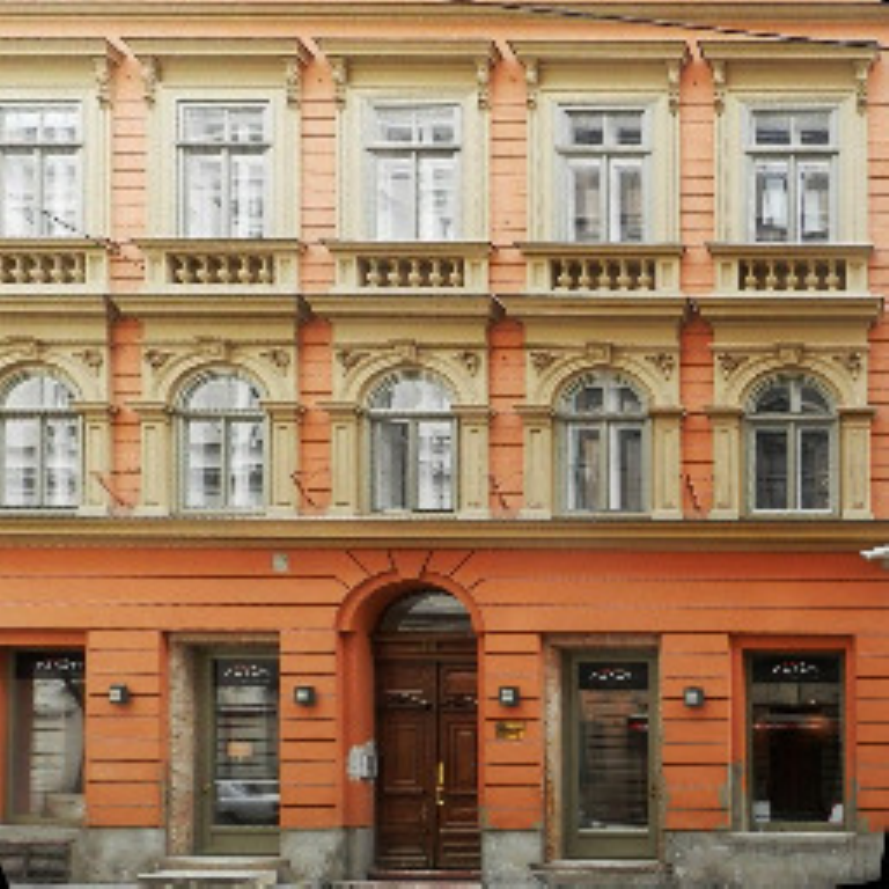} \\ 
\multicolumn{3}{c}{Original Image (Synthetic)} & \hspace{0.1in} &
\multicolumn{3}{c}{Original Image (Facades)} \vspace{5pt} \\
\includegraphics[width=0.12\textwidth]{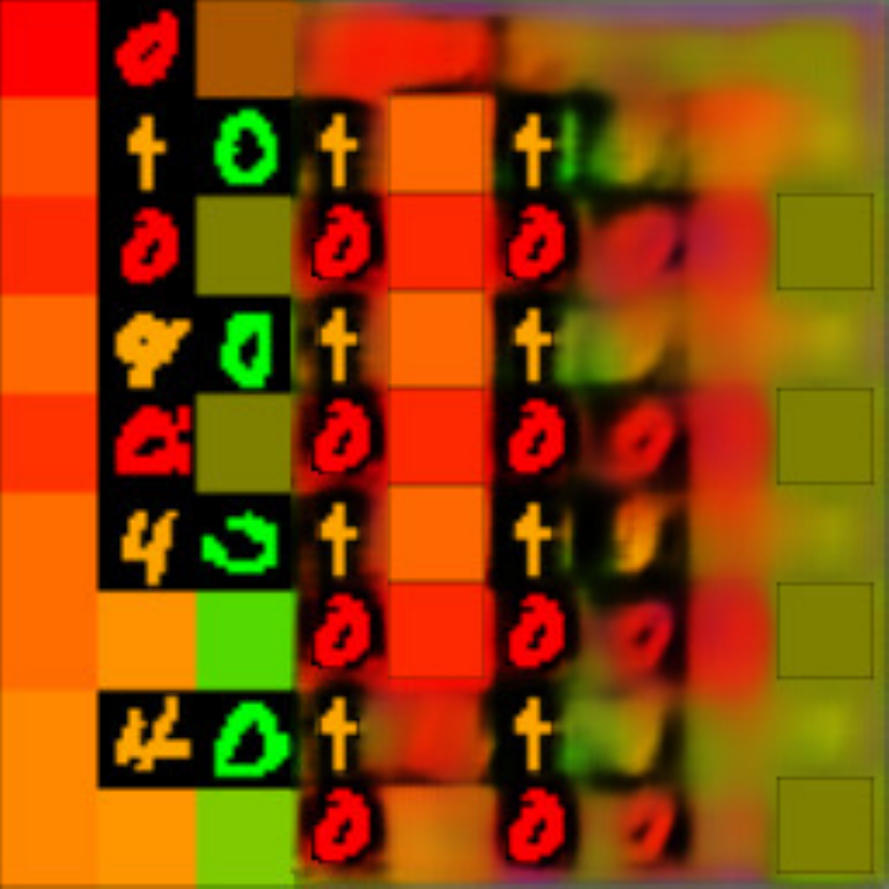} &
\includegraphics[width=0.12\textwidth]{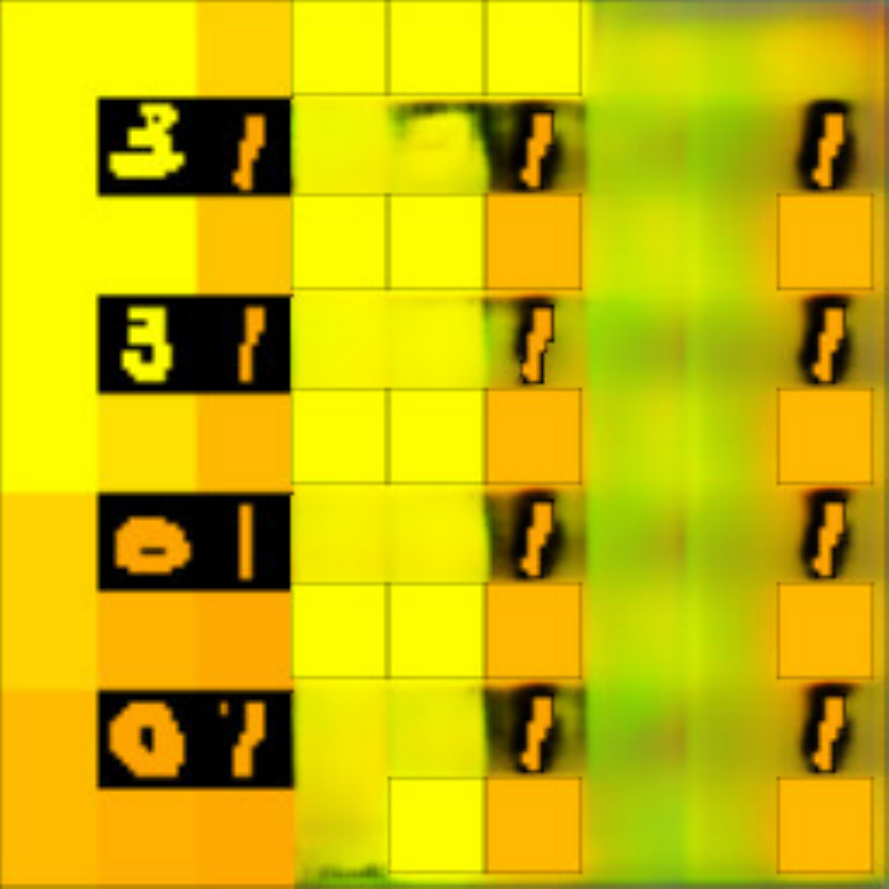} &
\includegraphics[width=0.12\textwidth]{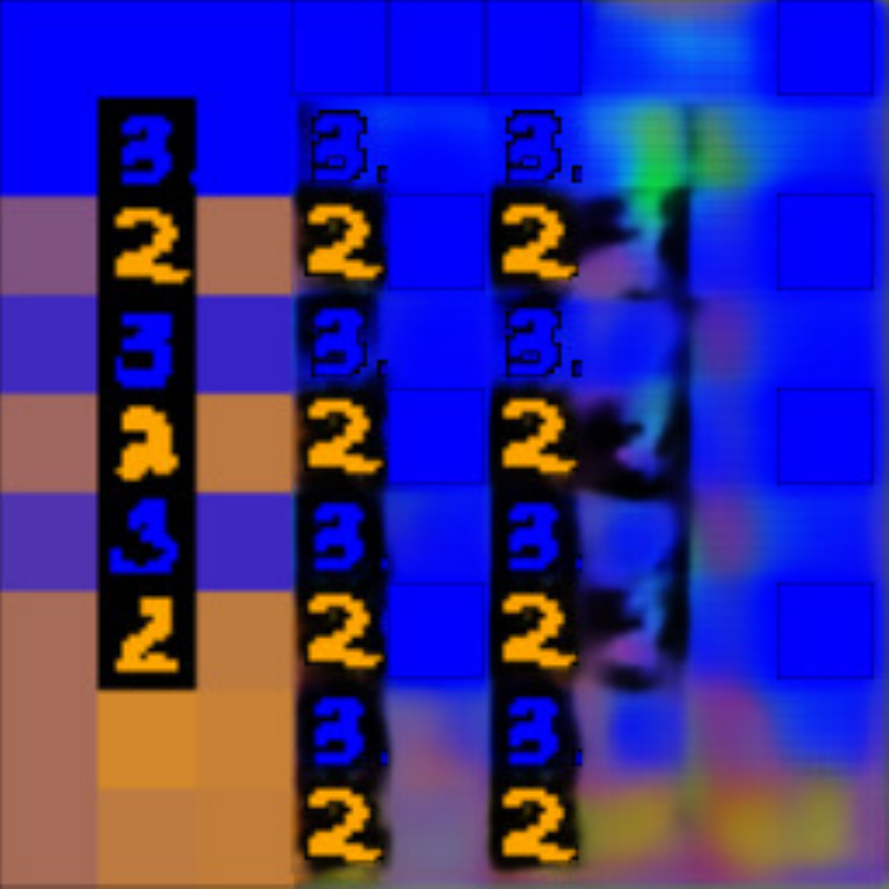} & \hspace{0.1in} &
\includegraphics[width=0.12\textwidth]{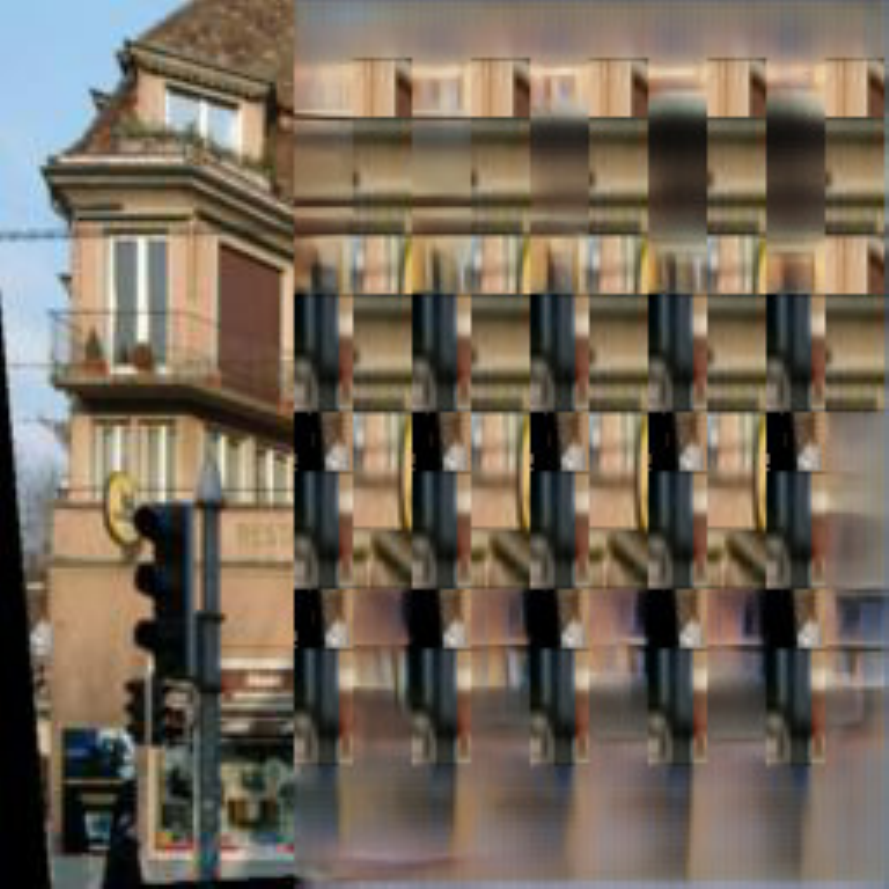} &
\includegraphics[width=0.12\textwidth]{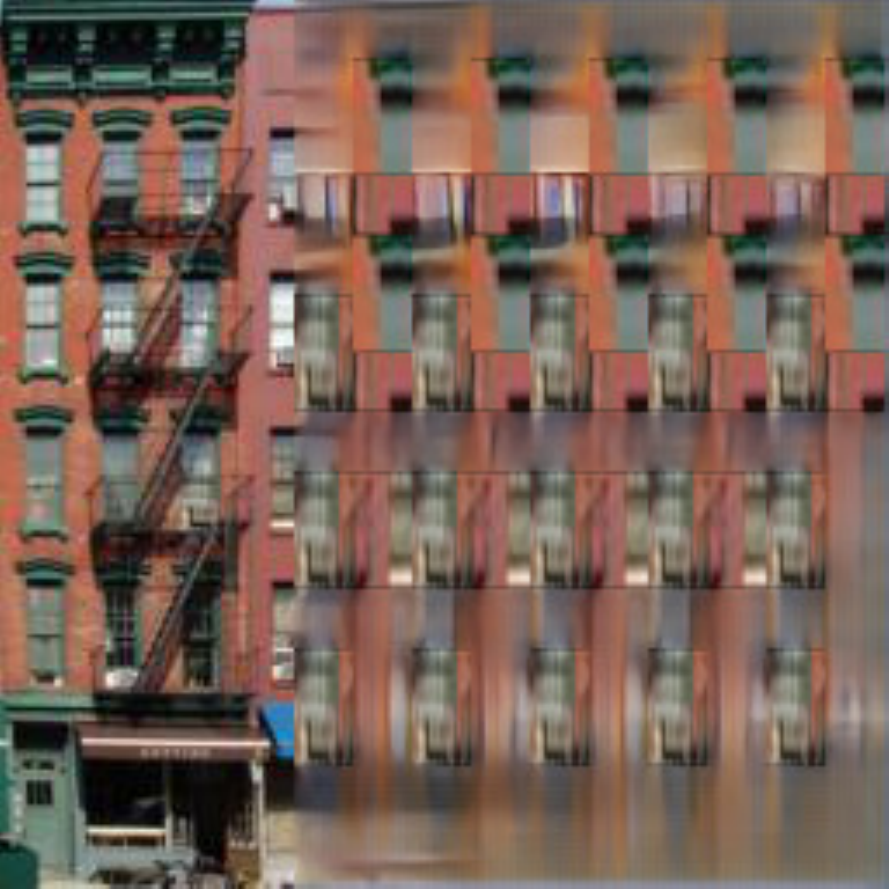} &
\includegraphics[width=0.12\textwidth]{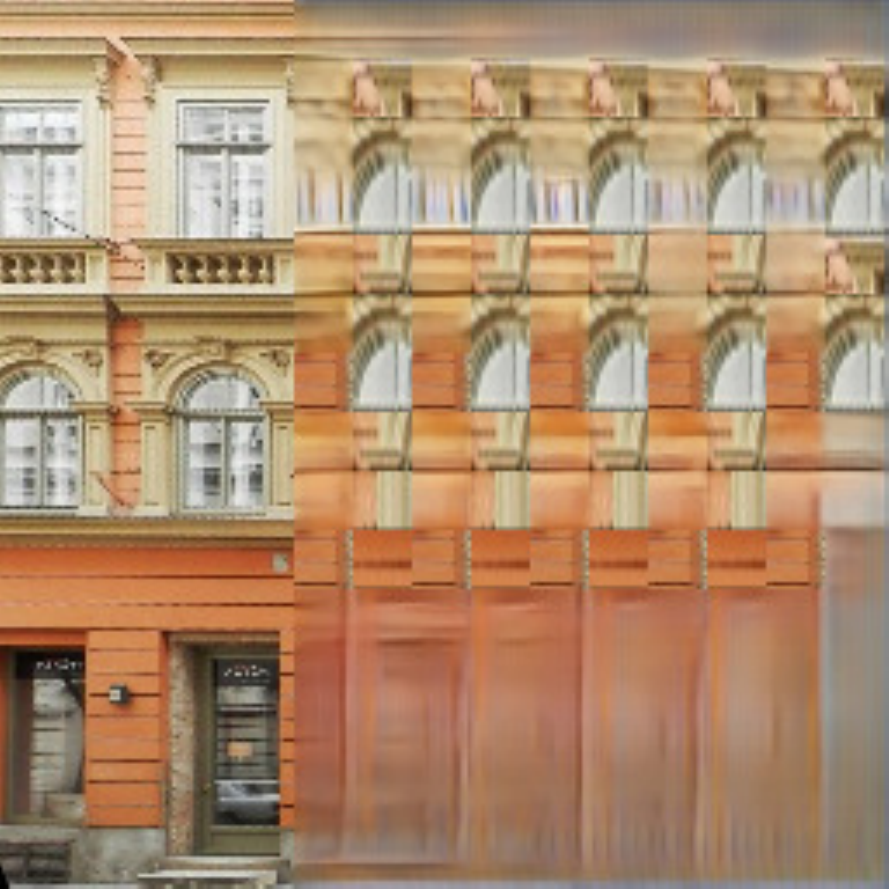} \\
\multicolumn{3}{c}{PS-GM (GLCIC, Synthetic)} & \hspace{0.1in} &
\multicolumn{3}{c}{PS-GM (GLCIC, Facades)} \vspace{5pt} \\
\includegraphics[width=0.12\textwidth]{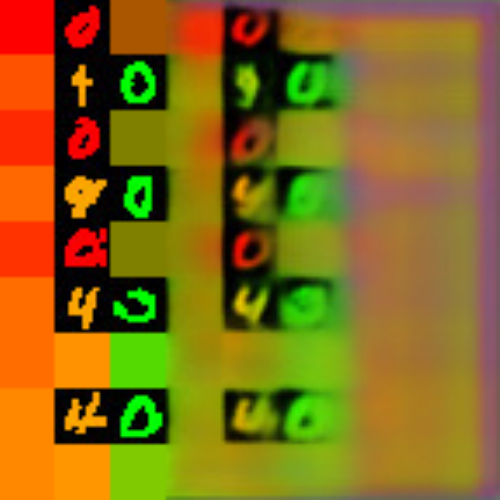} &
\includegraphics[width=0.12\textwidth]{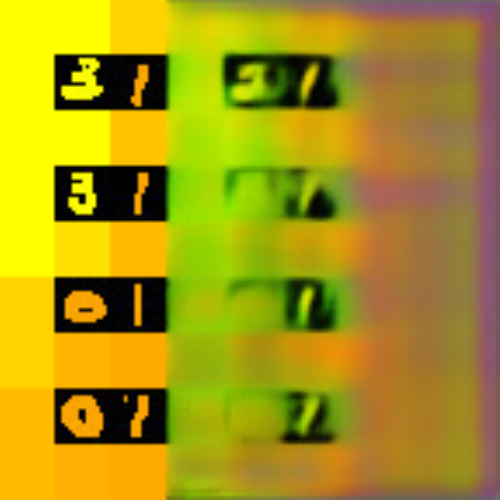} &
\includegraphics[width=0.12\textwidth]{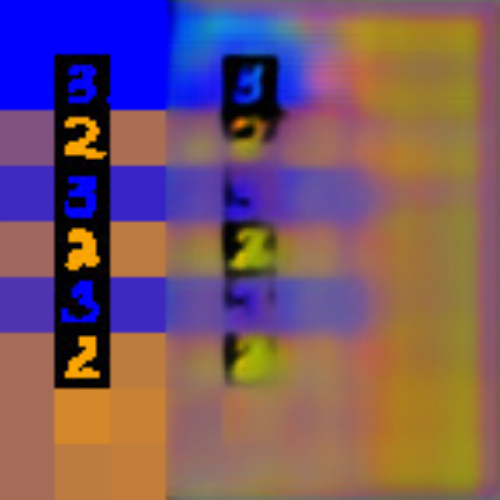} & \hspace{0.1in} &
\includegraphics[width=0.12\textwidth]{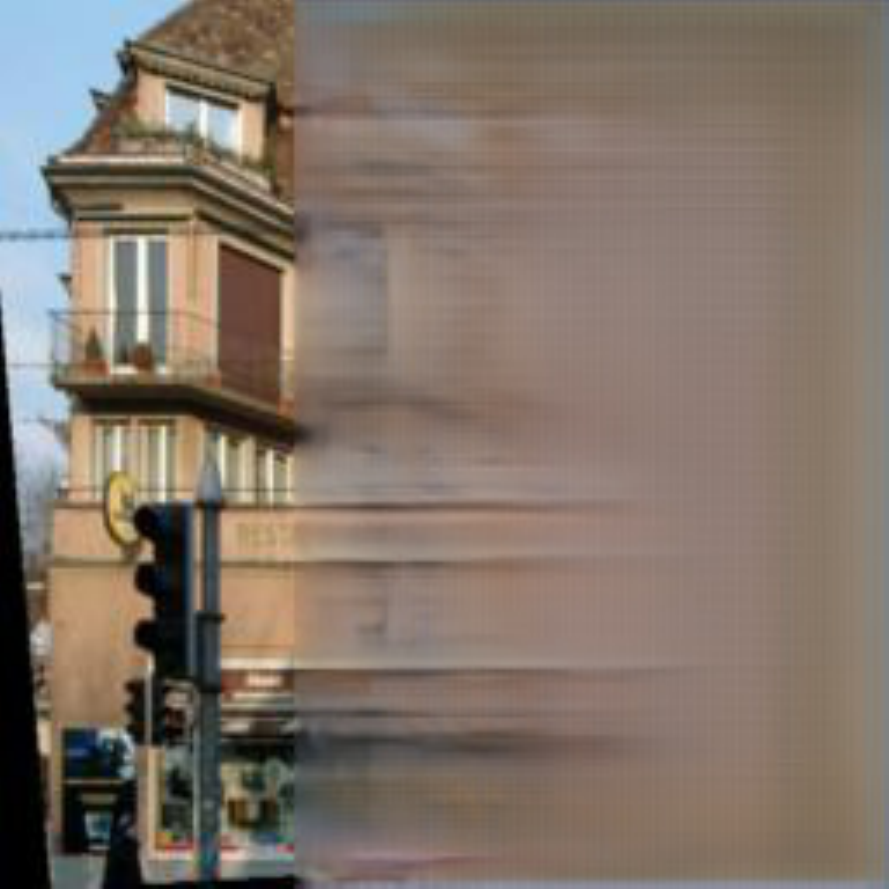} &
\includegraphics[width=0.12\textwidth]{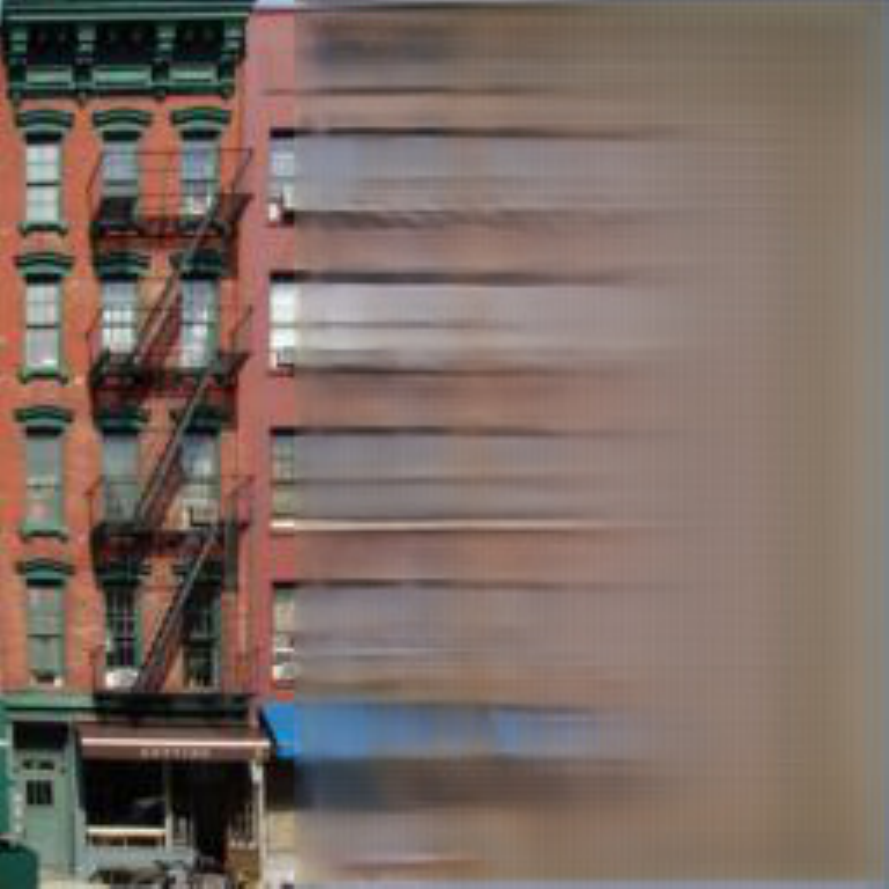} &
\includegraphics[width=0.12\textwidth]{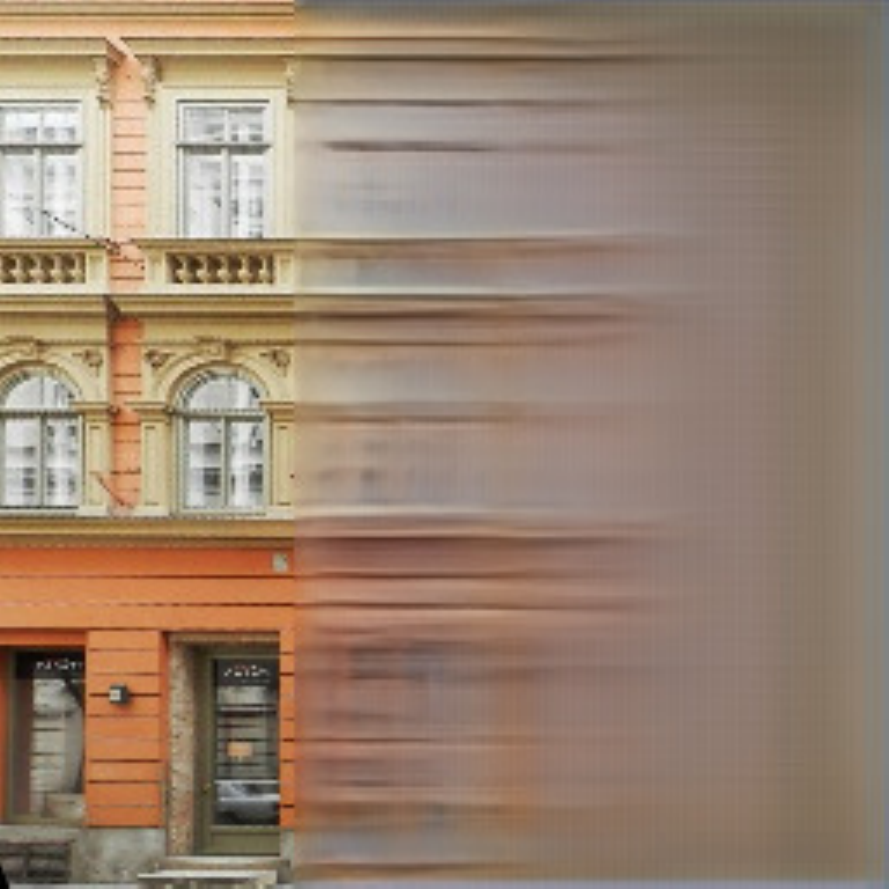} \\
\multicolumn{3}{c}{Baseline (GLCIC, Synthetic)} & \hspace{0.1in} &
\multicolumn{3}{c}{Baseline (GLCIC, Facades)}
\end{tabular}
\caption{Examples of images generated using our approach (PS-GM) and the baseline, using GLCIC for image completion.}
\label{fig:compexpimages}
\end{figure*}

We measure performance for PS-GM with the VED and the baseline VAE using the variational lower bound on the negative log-likelihood (NLL)~\cite{vaeloss} on a held-out test set. For our approach, we use the lower bound (\ref{eqn:variational2}),\footnote{Technically, $p_{\theta}(x\mid s_x,c_x)$ is lower bounded by the loss of the variational encoder-decoder).} which is the sum of the NLLs of the first and second stages; we report these NLLs separately as well. Figure~\ref{fig:genexpimages} shows examples of generated images. For PS-GM and SpatialGAN, we use Fr\'echet inception distance~\cite{heusel2017gans}. Table~\ref{tab:genexpresults} shows these metrics of both our approach and the baseline.

\paragraph{Discussion.}

The models based on our approach quantitatively improve over the respective baselines. The examples of images generated using our approach with VED completion appear to contain more structure than those generated using the baseline VAE. Similarly, the images generated using our approach with CycleGAN clearly capture more complex structure than the unbounded 2D repeating texture patterns captured by SpatialGAN.

\subsection{Image Completion}

\paragraph{Experimental setup.}

Second, we evaluated our approach PS-GM for image completion, on both our synthetic and the facades dataset. For this task, we compare using three image completion models: GLCIC~\cite{glcic}, CycleGAN~\cite{cyclegan}, and the VED architecture described in Section~\ref{sec:expgen}. GLCIC is a state-of-the-art image completion model. CycleGAN is a generic image-to-image transformer.  It uses unpaired training data, but we found that for our task, it outperforms approaches such as Pix2Pix~\cite{isola2017image} that take paired training data. For each model, we trained two versions:
{\setlength{\parskip}{0pt}
\begin{itemize}
\setlength{\itemsep}{0pt}
\item {\bf Our approach (PS-GM):} As described in Section~\ref{sec:model} (for image completion), given a partial image $x_{\text{part}}$, we use Algorithm~\ref{alg:synth} to synthesize a program $P_{\text{part}}$. We extrapolate $P_{\text{part}}$ to $\hat{P}=f(P_{\text{part}})$, and execute $\hat{P}$ to obtain a structure rendering $x_{\text{struct}}$. Finally, we train the image completion model (GLCIC, CycleGAN, or VED) to complete $x_{\text{struct}}$ to the original image $x^*$.
\item {\bf Baseline:} Given a partial image $x_{\text{part}}$, we train the image completion model (GLCIC, CycleGAN, or VED) to directly complete $x_{\text{part}}$ to the original image $x*$.
\end{itemize}}

\begin{table}
\centering
\begin{tabular}{lrrrr}
  \toprule
  \multirow{2}{*}{\textbf{Model}} & \multicolumn{2}{c}{\textbf{Synthetic}} & \multicolumn{2}{c}{\textbf{Facades}} \\
  & \multicolumn{1}{c}{PS-GM} & \multicolumn{1}{c}{BL} & \multicolumn{1}{c}{PS-GM} & \multicolumn{1}{c}{BL} \\
  \midrule
  GLCIC & {\bf 106.8} & 163.66 & {\bf 141.8} & 195.9 \\
  CycleGAN & {\bf 91.8} & 218.7 & {\bf 124.4} & 251.4 \\
  VED & {\bf 44570.4} & 52442.9 & 8755.4 & {\bf 8636.3}\\
  \bottomrule
\end{tabular}
\caption{Performance of our approach PS-GM versus the baseline (BL) for image completion. We report Fr\'echet distance for GAN-based models, and negative log-likelihood (NLL) for the VED.}
\label{tab:compexpresults}
\end{table}

\paragraph{Results.}

As in Section~\ref{sec:expgen}, we measure performance using Fr\'echet inception distance for GLCIC and CycleGAN, and negative log-likelihood (NLL) to evaluate the VED, reported on a held-out test set. We show these results in Table~\ref{tab:compexpresults}. We show examples of completed image using GLCIC in Figure~\ref{fig:compexpimages}. We show additional examples of completed images including those completed using CycleGAN and VED in Appendix~\ref{sec:expappendix}.

\paragraph{Discussion.}

Our approach PS-GM outperforms the baseline in every case except the VED on the facades dataset. We believe the last result is since both VEDs failed to learn any meaningful structure (see Figure~\ref{fig:expsyntheticappendix} in Appendix~\ref{sec:expappendix}).

A key reason why the baselines perform so poorly on the facades dataset is that the dataset is very small. Nevertheless, even on the synthetic dataset (which is fairly large), PS-GM substantially outperforms the baselines. Finally, generative models such as GLCIC are known to perform poorly far away from the edges of the provided partial image~\cite{glcic}. A benefit of our approach is that it provides the global context for a deep-learning based image completion model such as GLCIC to perform local completion.

%% file: conc.tex
\section{Conclusion}

We have proposed a new approach to generation that incorporates programmatic structure into state-of-the-art deep learning models. In our experiments, we have demonstrated the promise of our approach to improve generation of high-dimensional data with global structure that current state-of-the-art deep generative models have difficulty capturing.

%% file: appendix.tex
\section{Experimental Details}
\label{sec:expdetailsappendix}

\subsection{Synthetic Dataset}

To sample a random image, we started with a $9\times 9$ grid, where each grid cell is $16\times 16$ pixels. We randomly sample a program $P=((s_1,c_1),...,(s_k,c_k))$ (for $k=12$), where each perceptual component $c$ is a randomly selected MNIST image (downscaled to our grid cell size and colorized). To create correlations between different parts of $P$, we sample $(s_h,c_h)$ depending on $(s_1,c_1),...,(s_{h-1},c_{h-1})$. First, to sample each component $c_h$, we first sample latent properties of $c_h$ (i.e., its MNIST label $\{0,1,...,4\}$ and its color $\{\text{red},\text{blue},\text{orange},\text{green},\text{yellow}\}$). Second, we sample the parameters of $s_h$ conditional on these properties. To each of the 25 possible latent properties of $c_h$, we associate a discrete distribution over latent properties for later elements in the sequence, as well as a mean and standard deviation for each of the parameters of the corresponding sketch $s_h$.

We then render $P$ by executing each $(s_h,c_h)$ in sequence. However, when executing $(s_h,c_h)$, on each iteration $(i,j)$ of the for-loop, instead of rendering the sub-image $c_h$ at each position in the grid, we randomly sample another MNIST image $c_h^{(i,j)}$ with the same label as $c_h$, recolor $c_h^{(i,j)}$ to be the same color as $c_h$, and render $c_h^{(i,j)}$. By doing so, we introduce noise into the programmatic structure.

\subsection{Generation from Scratch}

\paragraph{PS-GM architecture.}

For the first stage of PS-GM (i.e., generating the program $P=(s,c)$), we use a 3-layer LSTM encoder $p_{\phi}(s,c\mid z)$ and a feedforward decoder $q_{\tilde{\phi}}(z\mid s,c)$.  The LSTM includes sequences of 13-dimensional vectors, of which 6 dimensions represent the structure of the for-loop being generated, and 7 dimensions are an encoding of the image to be rendered.  The image compression was performed via a convolutional architecture with 2 convolutional layers for encoding and 3 deconvolutional layers for decoding.

For the second stage of PS-GM (i.e., completing the structure rendering $x_{\text{struct}}$ into an image $x$), we use a VED; the encoder $q_{\theta}(w\mid x_{\text{struct}})$ is a CNN with 4 layers, and the decoder $p_{\theta}(x\mid w)$ is a transpose CNN with 6 layers.  The CycleGAN model has a discriminator with 3 convolutional layers and a generator which uses transfer learning by employing the pre-trained ResNet architecture.

\paragraph{Baseline architecture.}

The architecture of the baseline is a vanilla VAE with the same as the architecture as the VED we used for the second state of PS-GM, except the input to the encoder is the original training image $x$ instead of the structure rendering $x_{\text{struct}}$.  The baselines with CycleGAN also use the same architecture as PS-GM with CycleGAN/GLCIC.  The Spatial GAN was trained with 5 layers each in the generative/discriminative layer, and 60-dimensional global and 3-dimensional periodic latent vectors.

\subsection{Image completion.}

\paragraph{PS-GM architecture.}

For the first stage of PS-GM for completion (extrapolation of the program from a partial image to a full image), we use a feedforward network with three layers.  For the second stage of completion via VAE, we use a convolutional/deconvolutional architecture.  The encoder is a CNN with 4 layers, and the decoder is a transpose CNN with 6 layers.  As was the case in generation, the CycleGAN model has a discriminator with 3 convolutional layers and a generator which uses transfer learning by employing the pre-trained ResNet architecture.

\paragraph{Baseline architecture.}

For the baseline VAE architecture, we used a similar architecture to the PS-GM completion step (4 convolutional and 6 deconvolutional layers).  The only difference was the input, which was a partial image rather than an image rendered with structure.  The CycleGAN architecture was similar to that used in PS-GM (although it mapped partial images to full images rather than partial images with structure to full images).

\section{Additional Results}
\label{sec:expappendix}

In Figure~\ref{fig:expfacadesappendix}, we show examples of how our image completion pipeline is applied to the facades dataset, and in Figure~\ref{fig:expsyntheticappendix}, we show examples of how our image completion pipeline is applied to our synthetic dataset.

\begin{figure*}[ht]
  \centering
  \begin{tabular}{rccc}
    Original Image & \includegraphics[width=0.1\textwidth]{images-facades-full1.pdf} & \includegraphics[width=0.1\textwidth]{images-facades-full2.pdf} & \includegraphics[width=0.1\textwidth]{images-facades-full3.pdf} \\ 
    Partial Image & \includegraphics[width=0.1\textwidth]{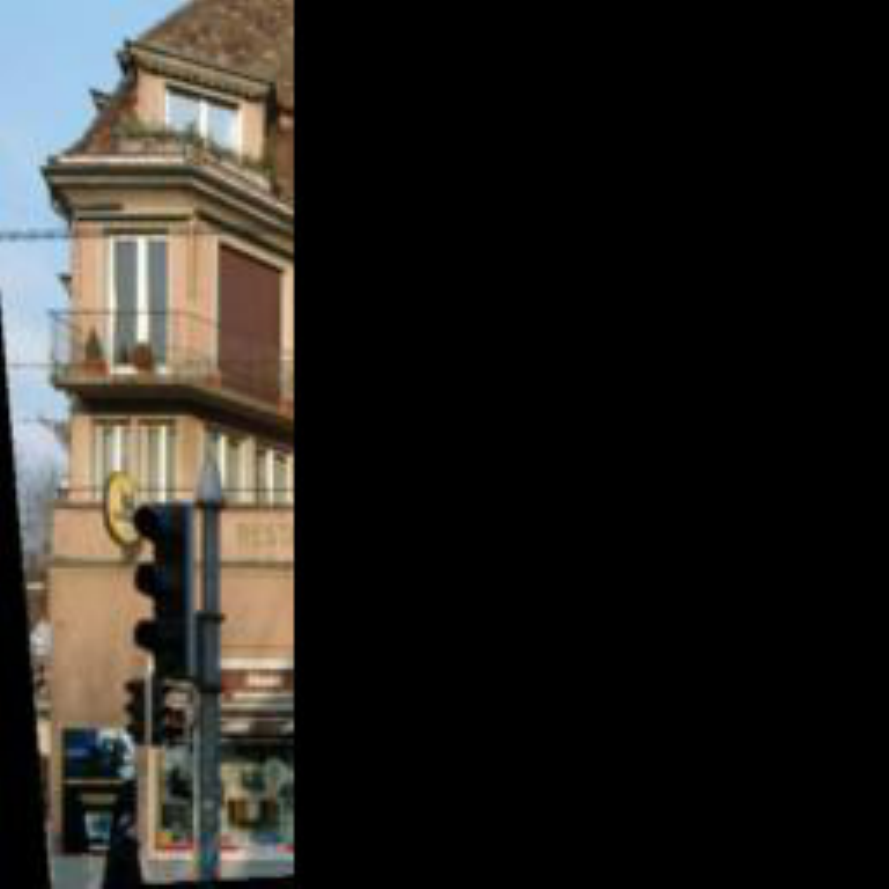}  & \includegraphics[width=0.1\textwidth]{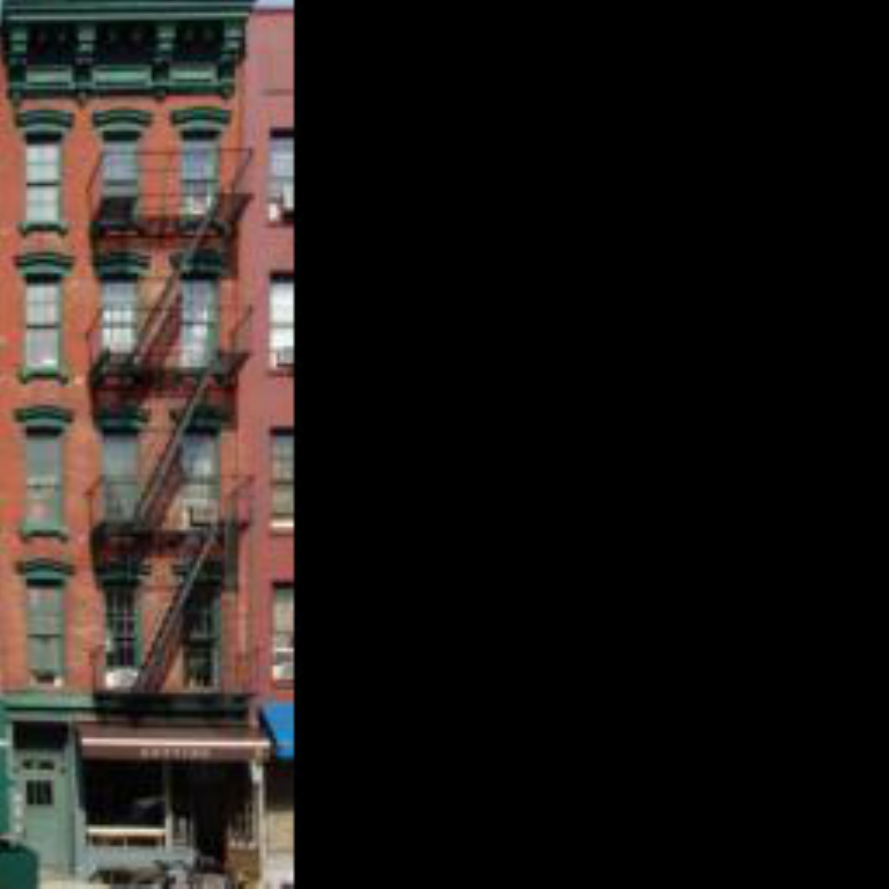} & \includegraphics[width=0.1\textwidth]{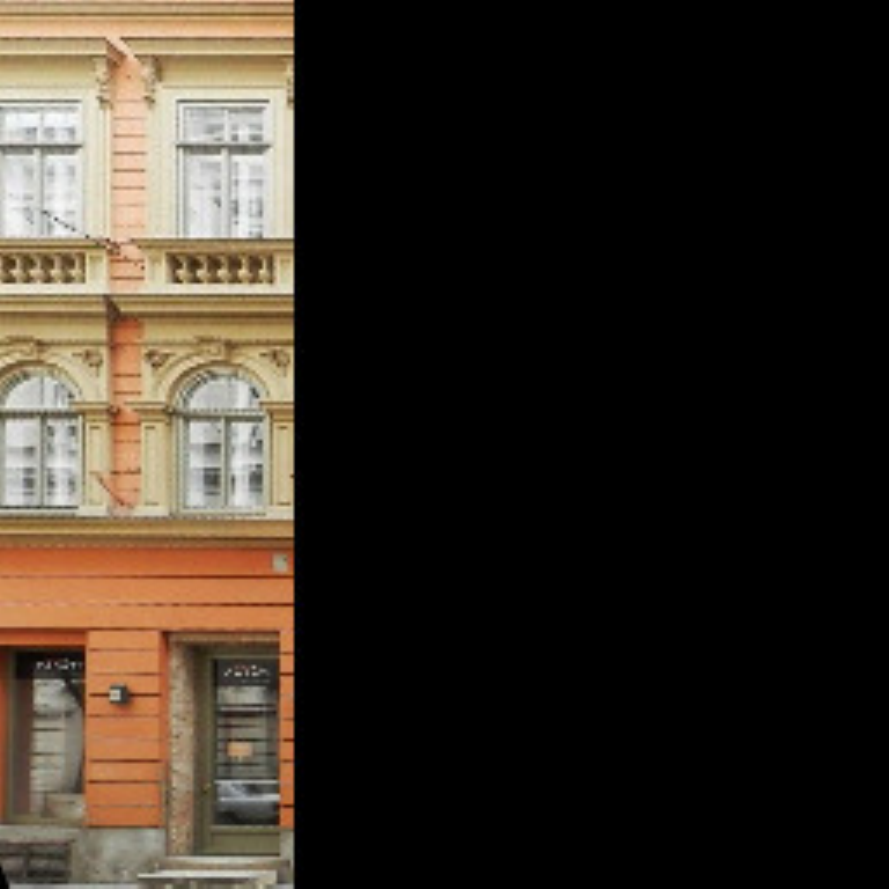} \\
    Structure Rendering & \includegraphics[width=0.1\textwidth]{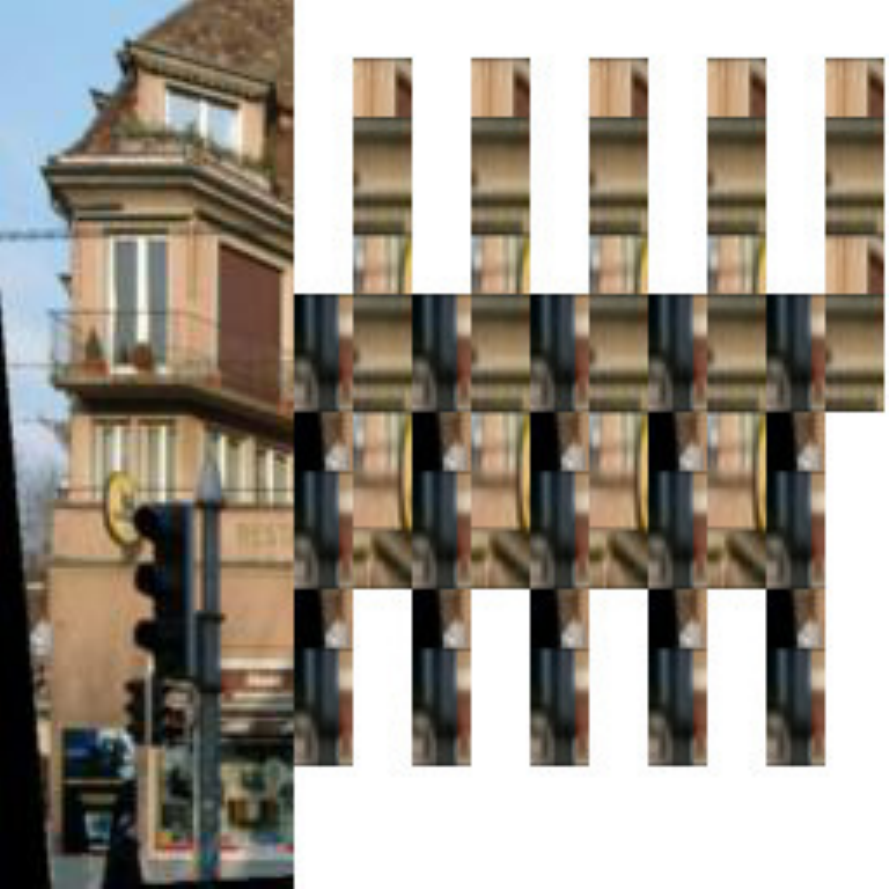}  & \includegraphics[width=0.1\textwidth]{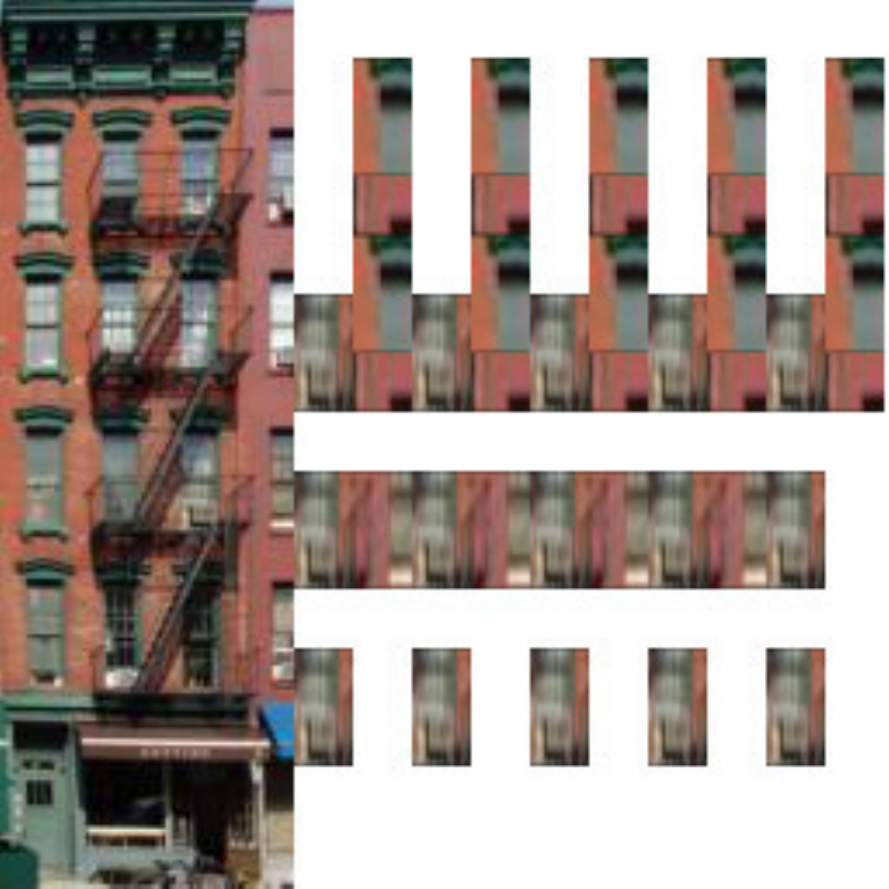} & \includegraphics[width=0.1\textwidth]{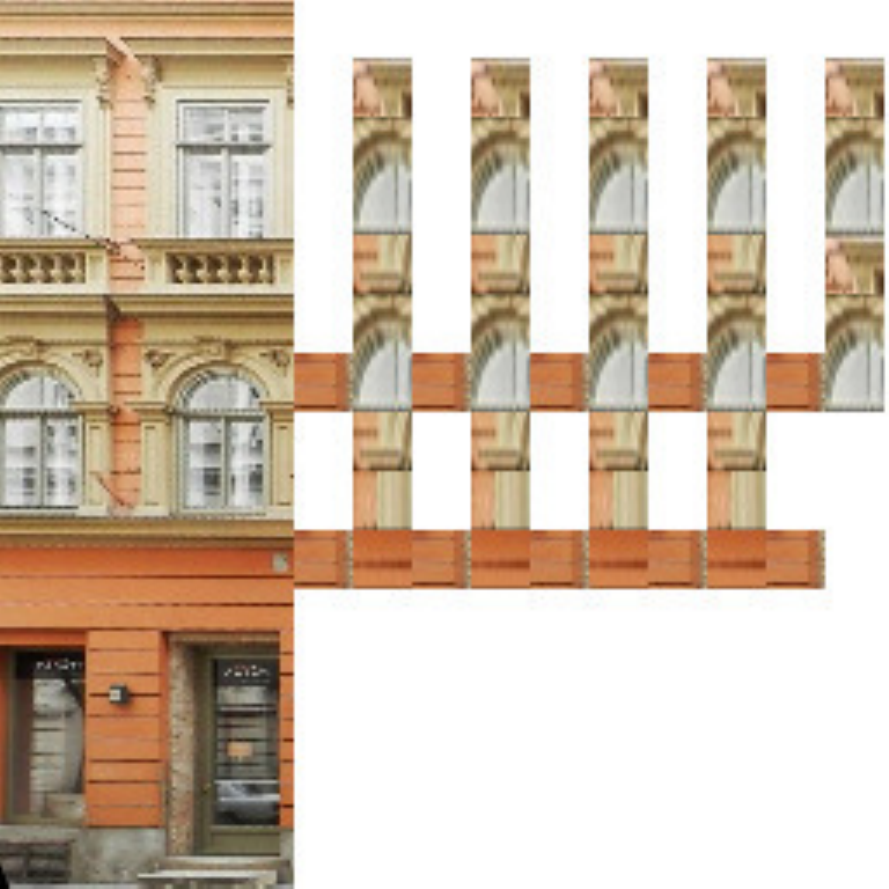} \\
    PS-GM (GLCIC) & \includegraphics[width=0.1\textwidth]{images-facades-glcicprog1.pdf}  & \includegraphics[width=0.1\textwidth]{images-facades-glcicprog2.pdf} & \includegraphics[width=0.1\textwidth]{images-facades-glcicprog3.pdf} \\
    Baseline (GLCIC) & \includegraphics[width=0.1\textwidth]{images-facades-glcic1.pdf}  & \includegraphics[width=0.1\textwidth]{images-facades-glcic2.pdf} & \includegraphics[width=0.1\textwidth]{images-facades-glcic3.pdf} \\
    PS-GM (CycleGAN) & \includegraphics[width=0.1\textwidth]{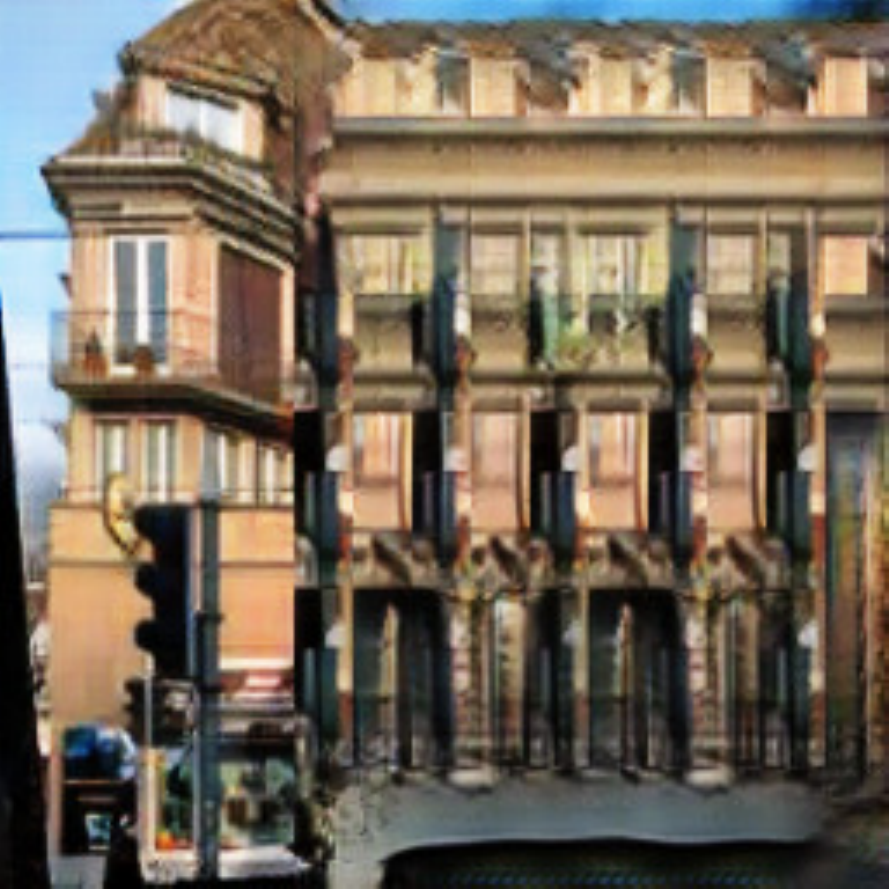}  & \includegraphics[width=0.1\textwidth]{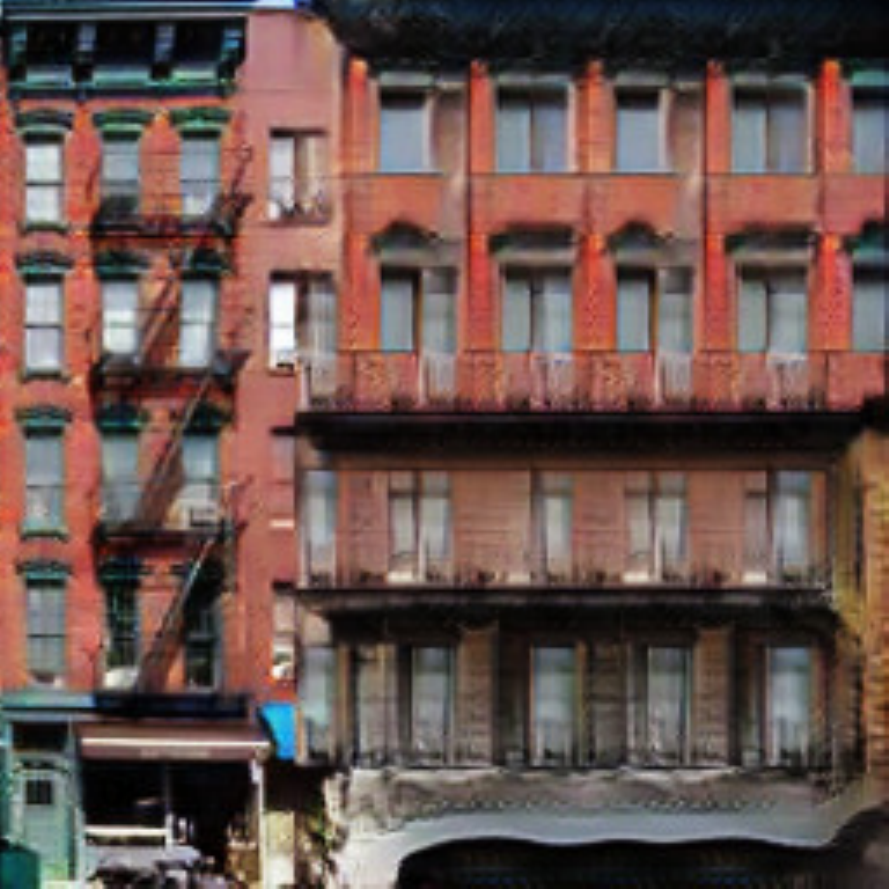} & \includegraphics[width=0.1\textwidth]{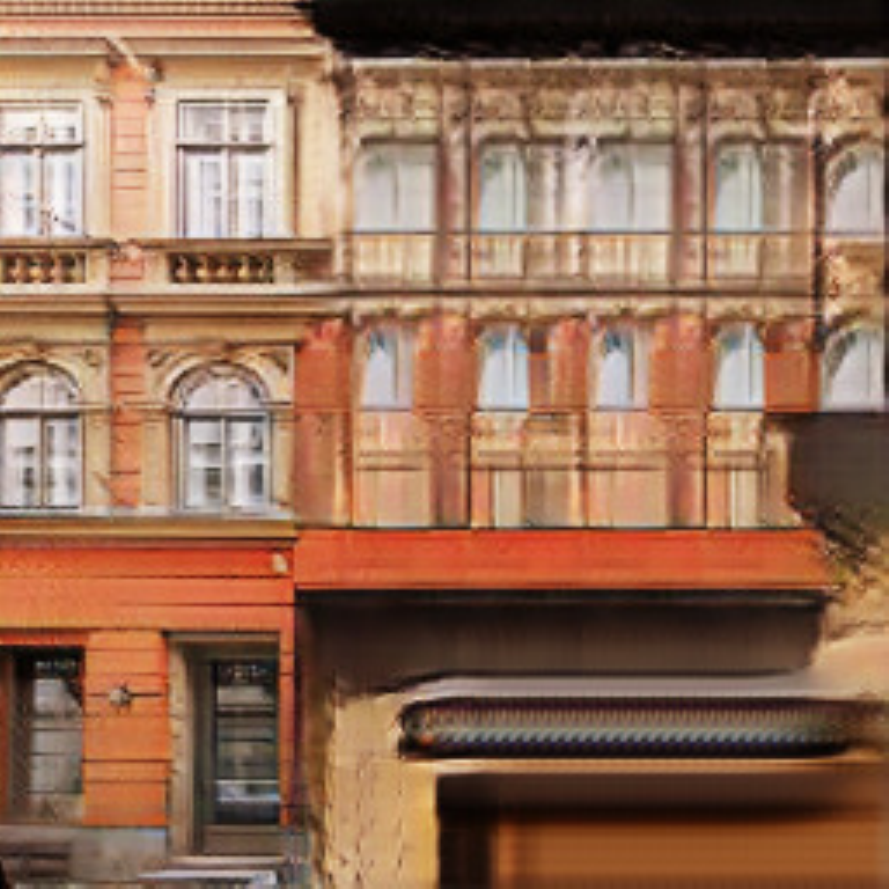} \\
    Baseline (CycleGAN) & \includegraphics[width=0.1\textwidth]{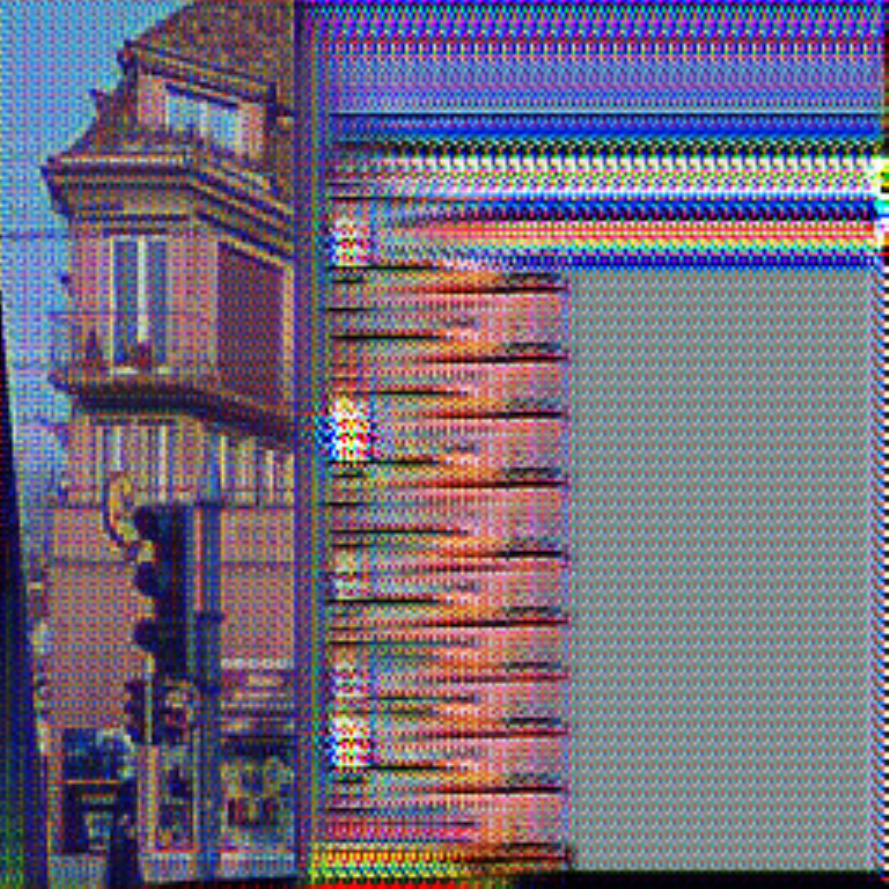}  & \includegraphics[width=0.1\textwidth]{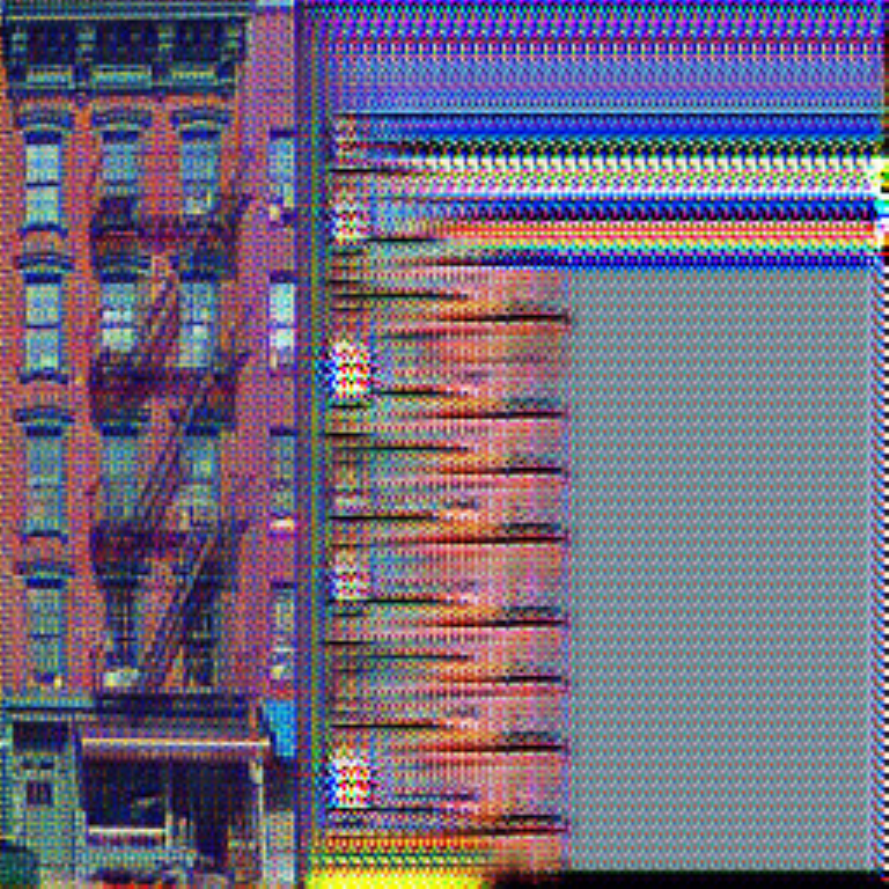} & \includegraphics[width=0.1\textwidth]{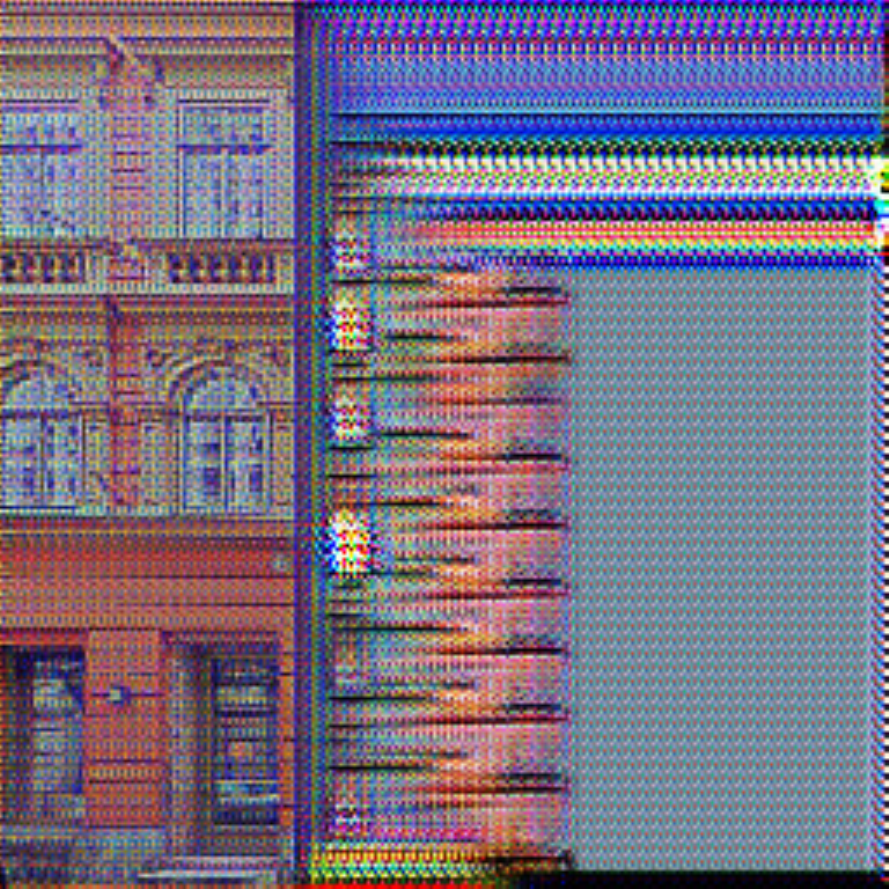} \\
    PS-GM (VED) & \includegraphics[width=0.1\textwidth]{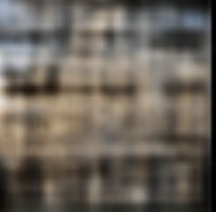}  & \includegraphics[width=0.1\textwidth]{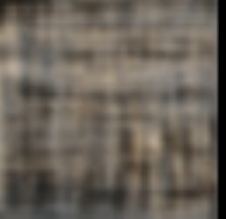} & \includegraphics[width=0.1\textwidth]{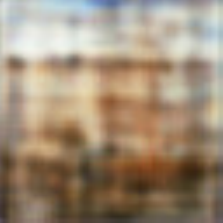} \\
    Baseline (VED) & \includegraphics[width=0.1\textwidth]{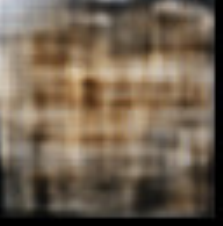}  & \includegraphics[width=0.1\textwidth]{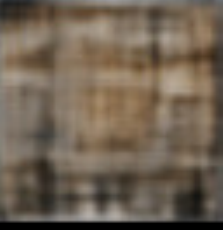} & \includegraphics[width=0.1\textwidth]{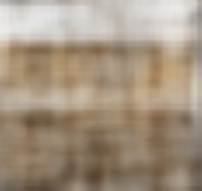} \\
\end{tabular}
  \caption{Examples of our image completion pipeline on the facades dataset.}
  \label{fig:expfacadesappendix}
\end{figure*}

\begin{figure*}[ht]
  \centering
  \begin{tabular}{rccc}
    Original Image & \includegraphics[width=0.1\textwidth]{images-synth-full1.pdf} & \includegraphics[width=0.1\textwidth]{images-synth-full2.pdf} & \includegraphics[width=0.1\textwidth]{images-synth-full3.pdf} \\ 
    Partial Image & \includegraphics[width=0.1\textwidth]{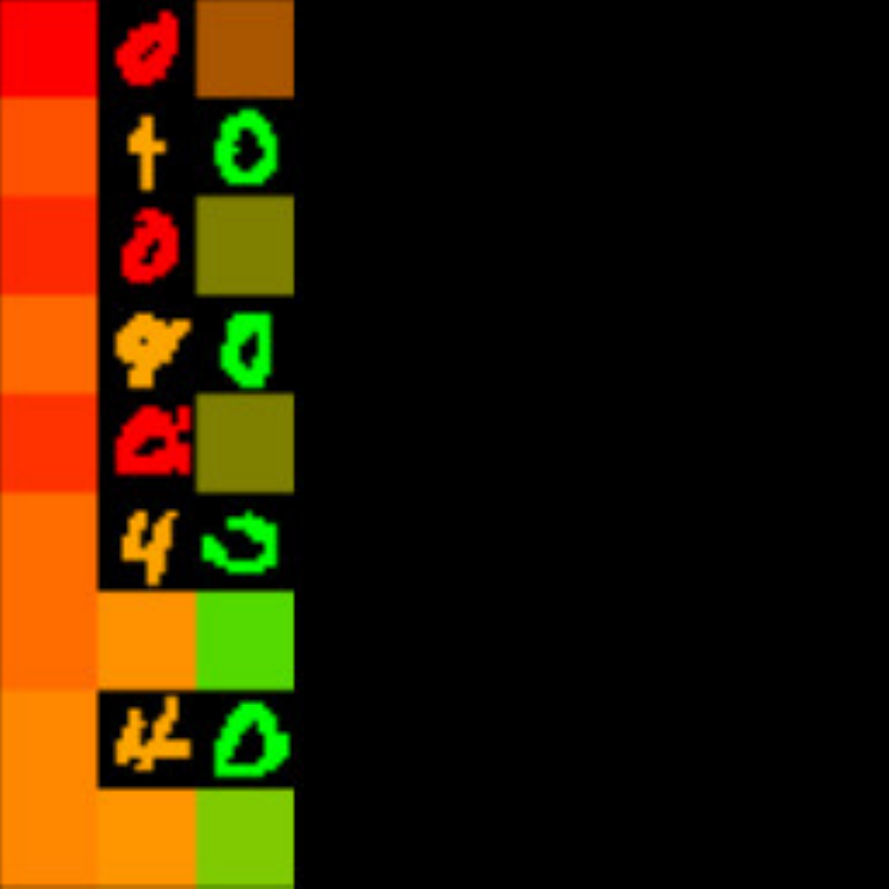}  & \includegraphics[width=0.1\textwidth]{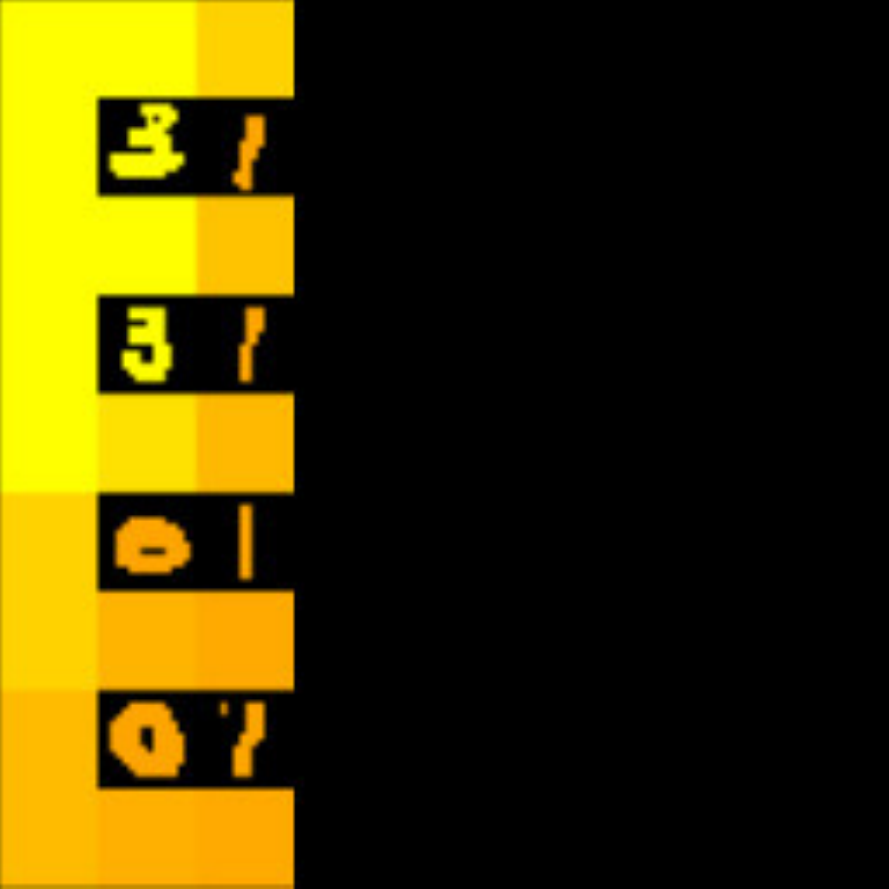} & \includegraphics[width=0.1\textwidth]{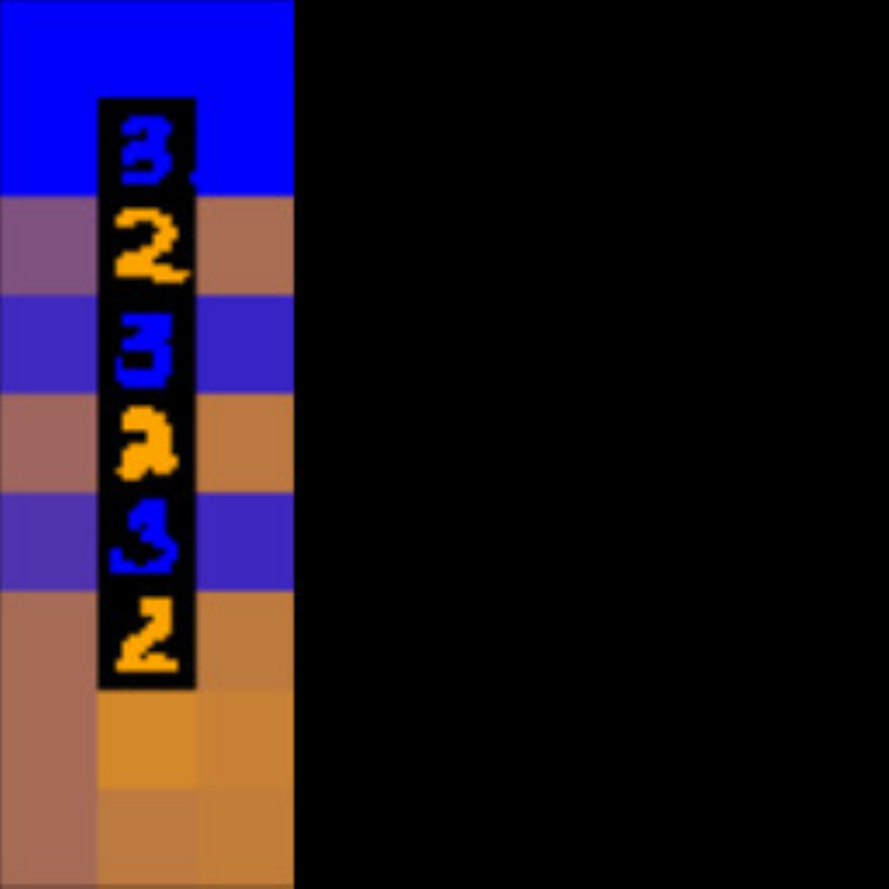} \\
    Structure Rendering (Partial) & \includegraphics[width=0.1\textwidth]{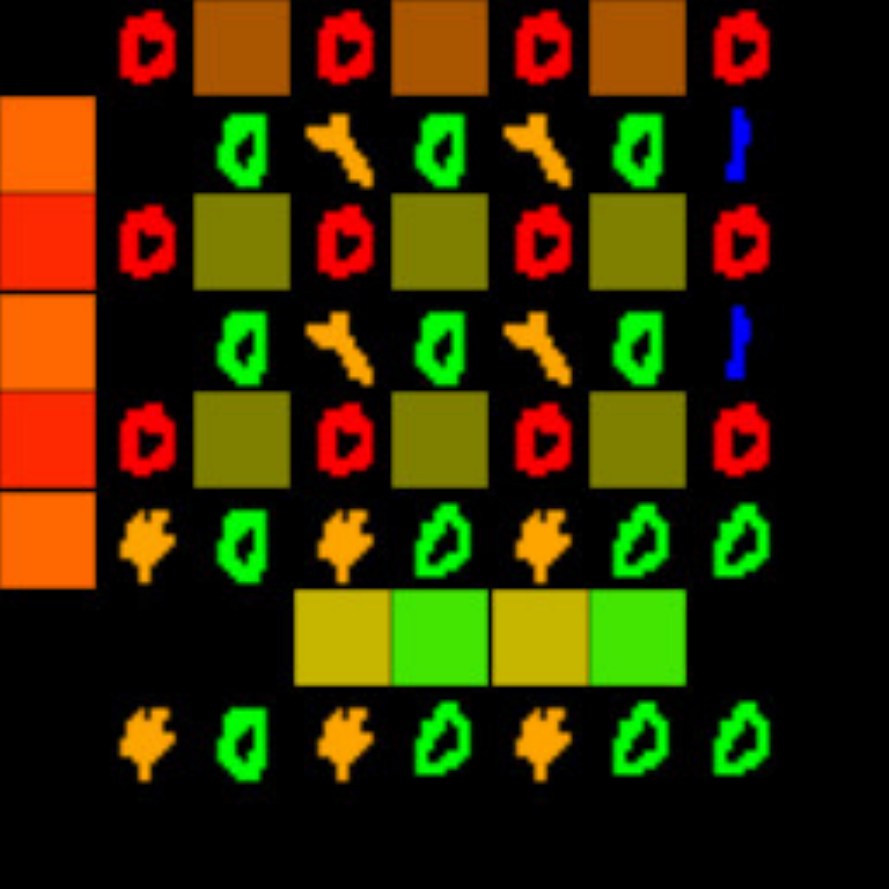}  & \includegraphics[width=0.1\textwidth]{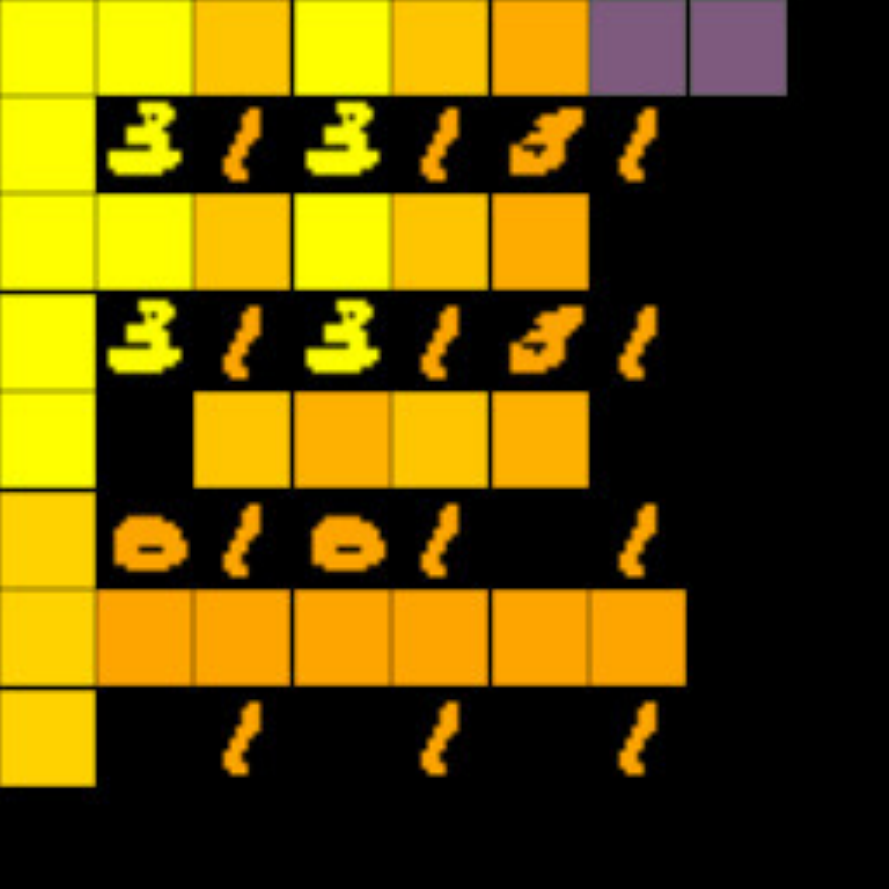} & \includegraphics[width=0.1\textwidth]{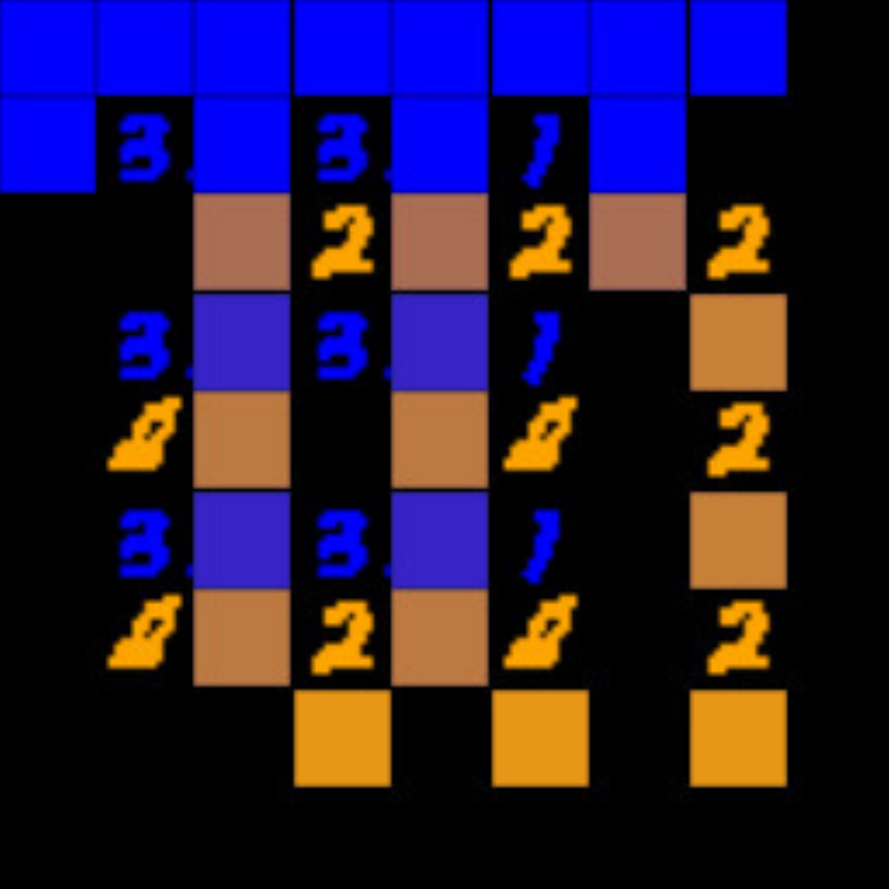} \\
    Structure Rendering (Extrapolated) & \includegraphics[width=0.1\textwidth]{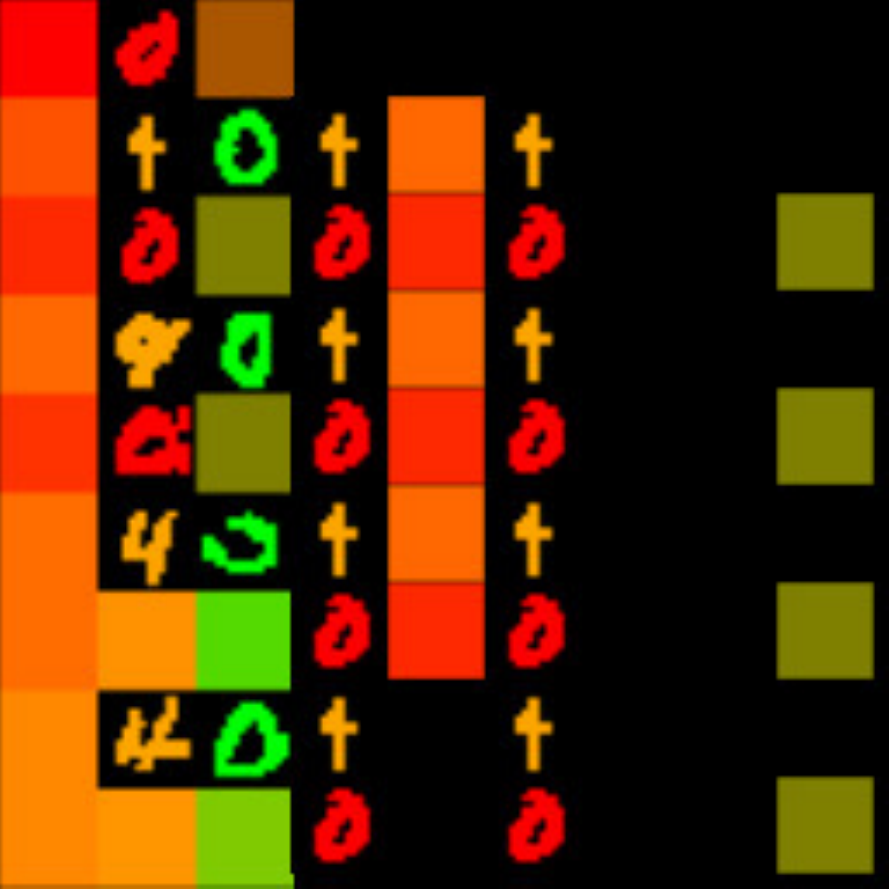}  & \includegraphics[width=0.1\textwidth]{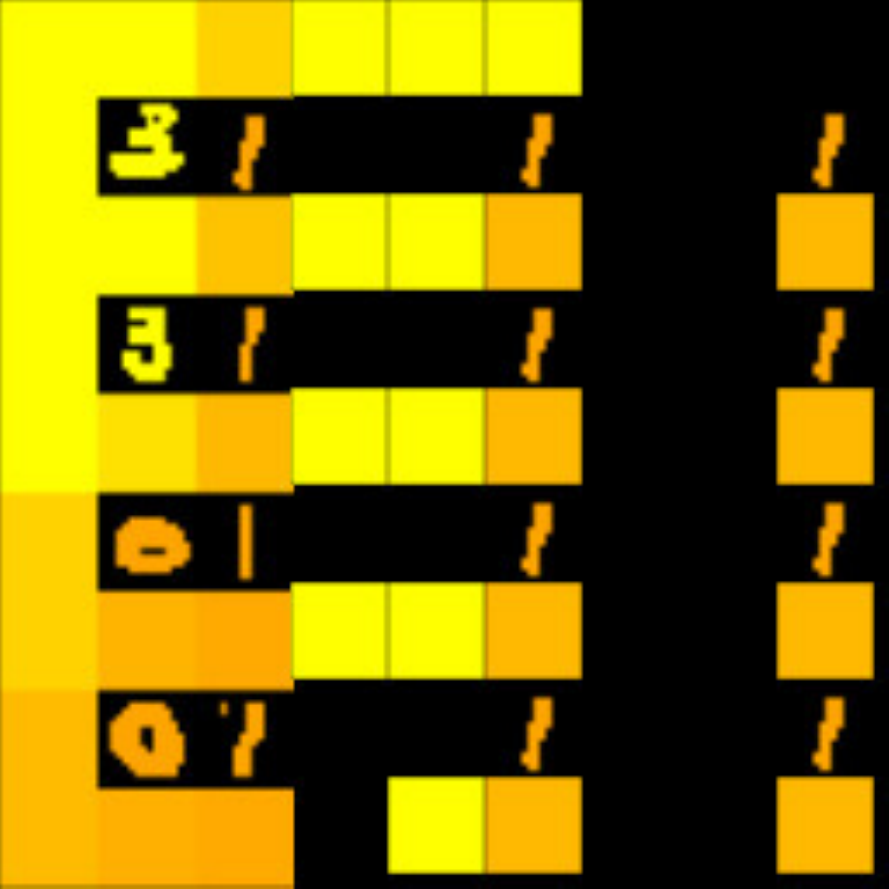} & \includegraphics[width=0.1\textwidth]{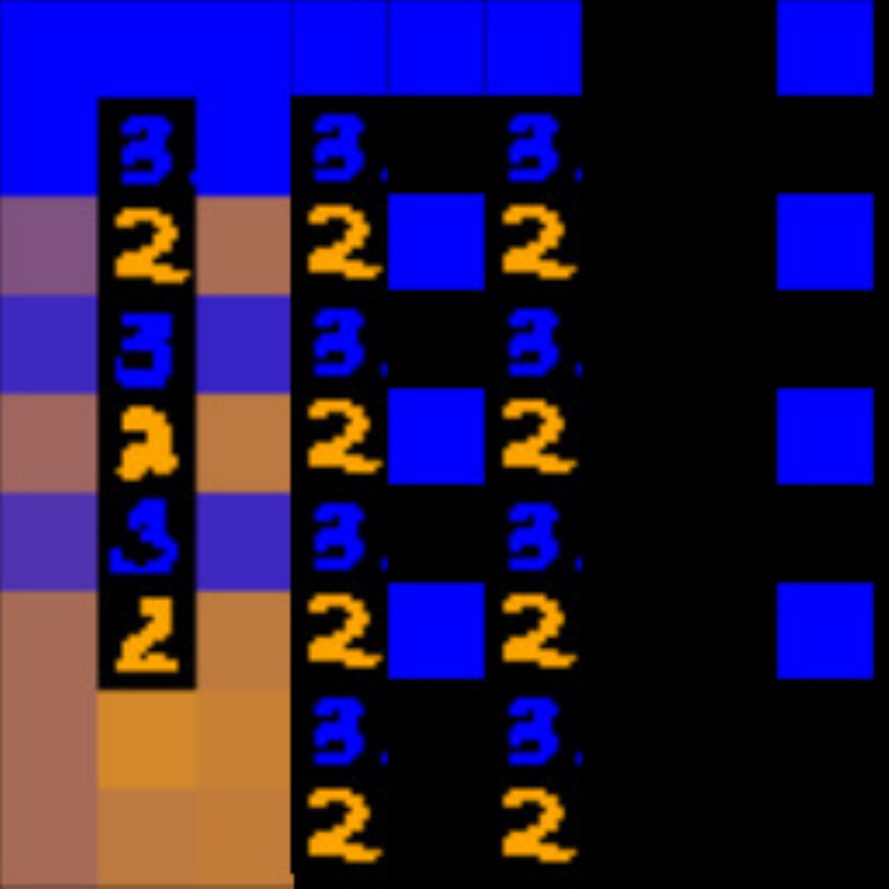} \\
    PS-GM (GLCIC) & \includegraphics[width=0.1\textwidth]{images-synth-glcicprog1.pdf}  & \includegraphics[width=0.1\textwidth]{images-synth-glcicprog2.pdf} & \includegraphics[width=0.1\textwidth]{images-synth-glcicprog3.pdf} \\
    Baseline (GLCIC) & \includegraphics[width=0.1\textwidth]{images-synth-glcic1.pdf}  & \includegraphics[width=0.1\textwidth]{images-synth-glcic2.pdf} & \includegraphics[width=0.1\textwidth]{images-synth-glcic3.pdf} \\
    PS-GM (CycleGAN) & \includegraphics[width=0.1\textwidth]{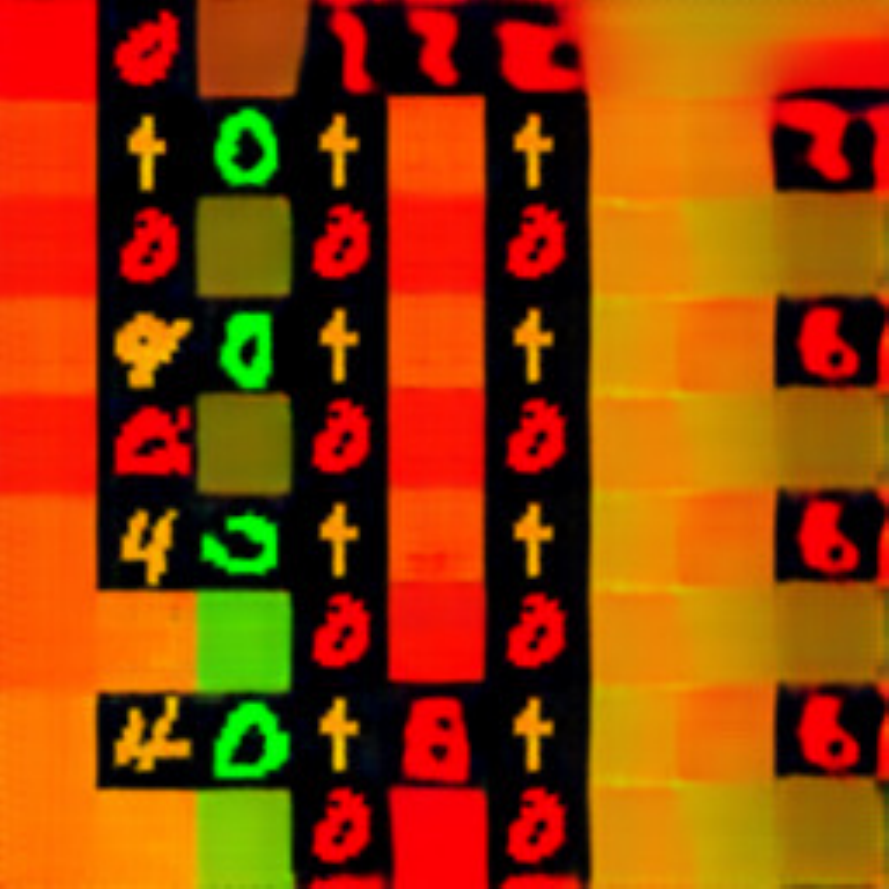}  & \includegraphics[width=0.1\textwidth]{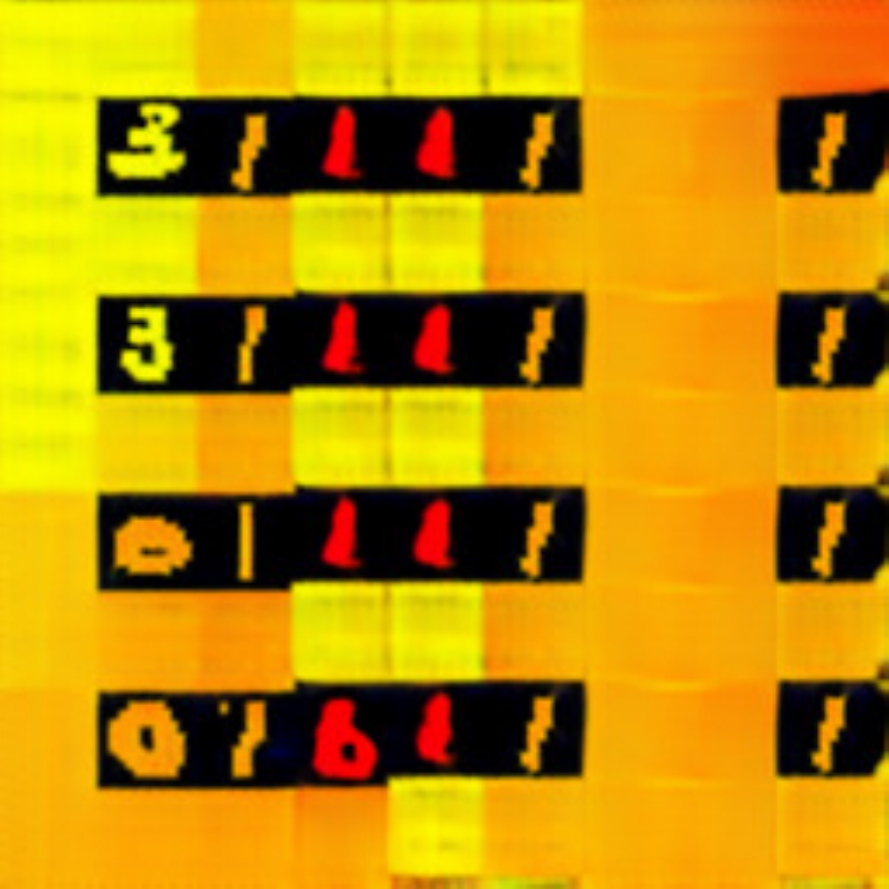} & \includegraphics[width=0.1\textwidth]{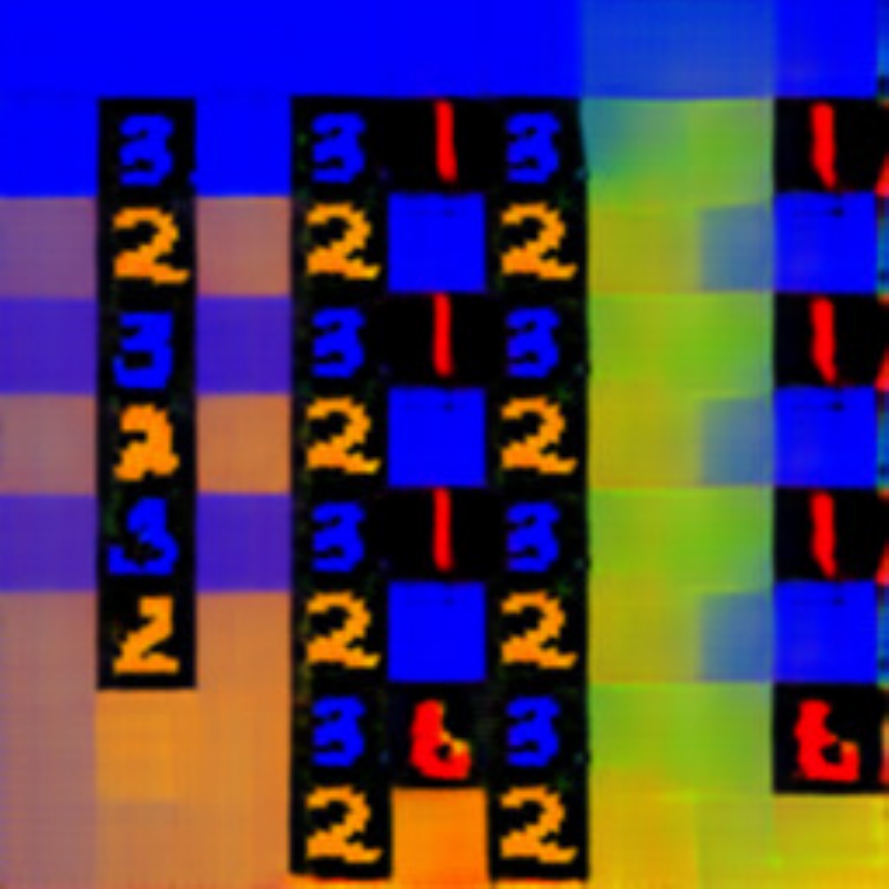} \\
    Baseline (CycleGAN) & \includegraphics[width=0.1\textwidth]{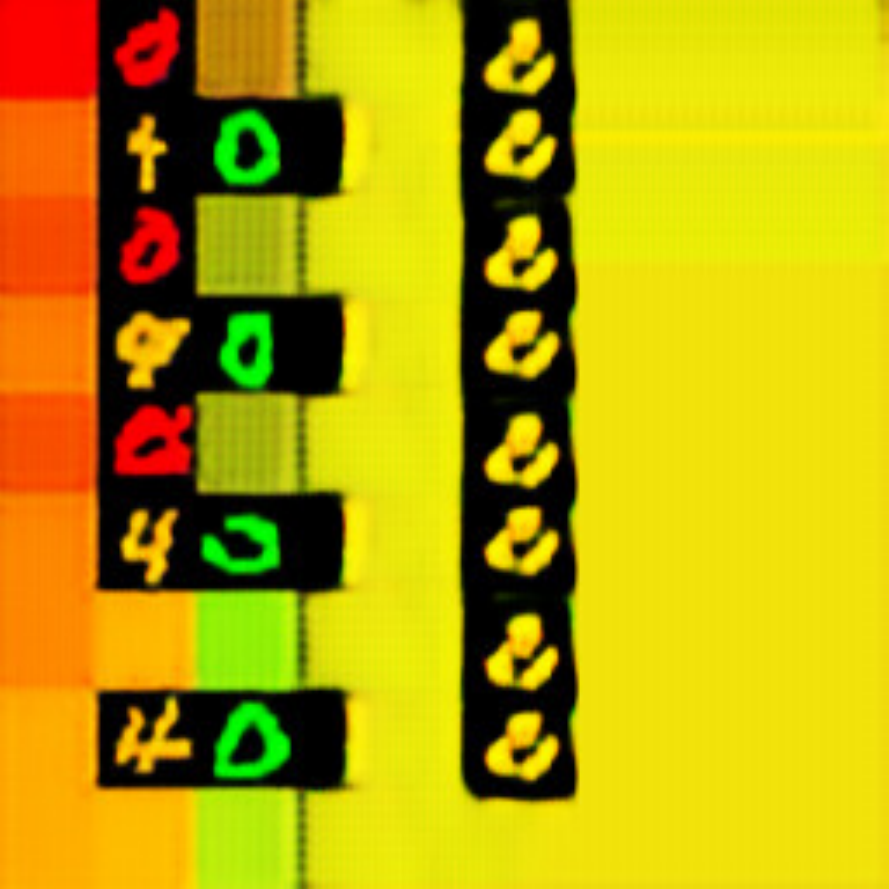}  & \includegraphics[width=0.1\textwidth]{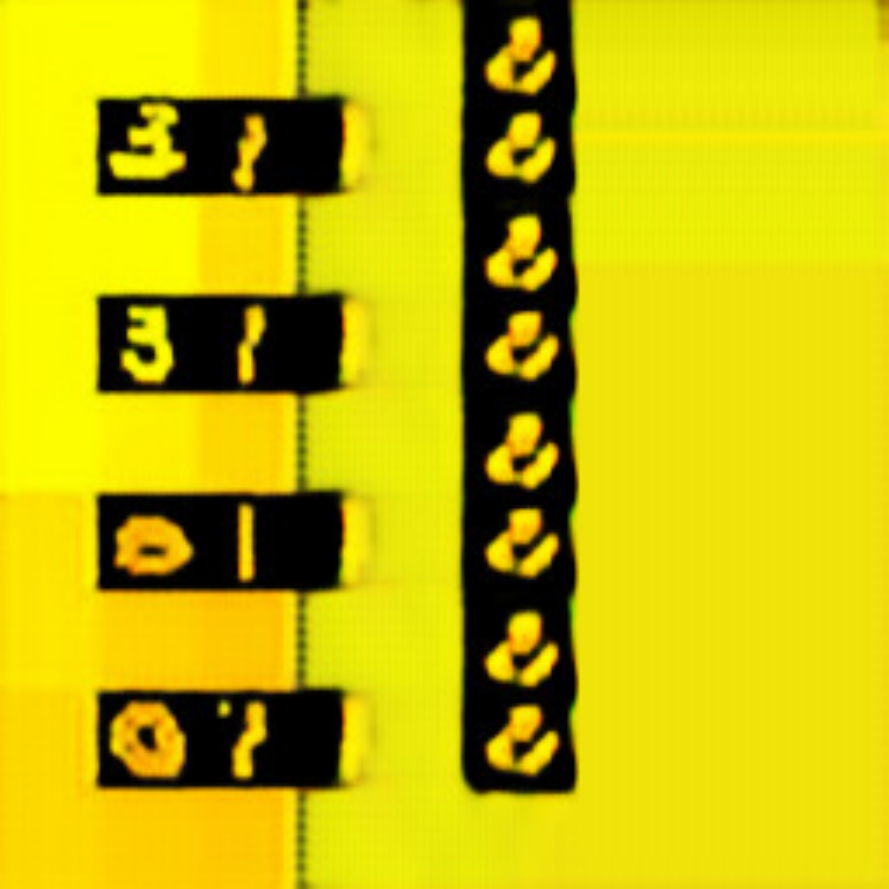} & \includegraphics[width=0.1\textwidth]{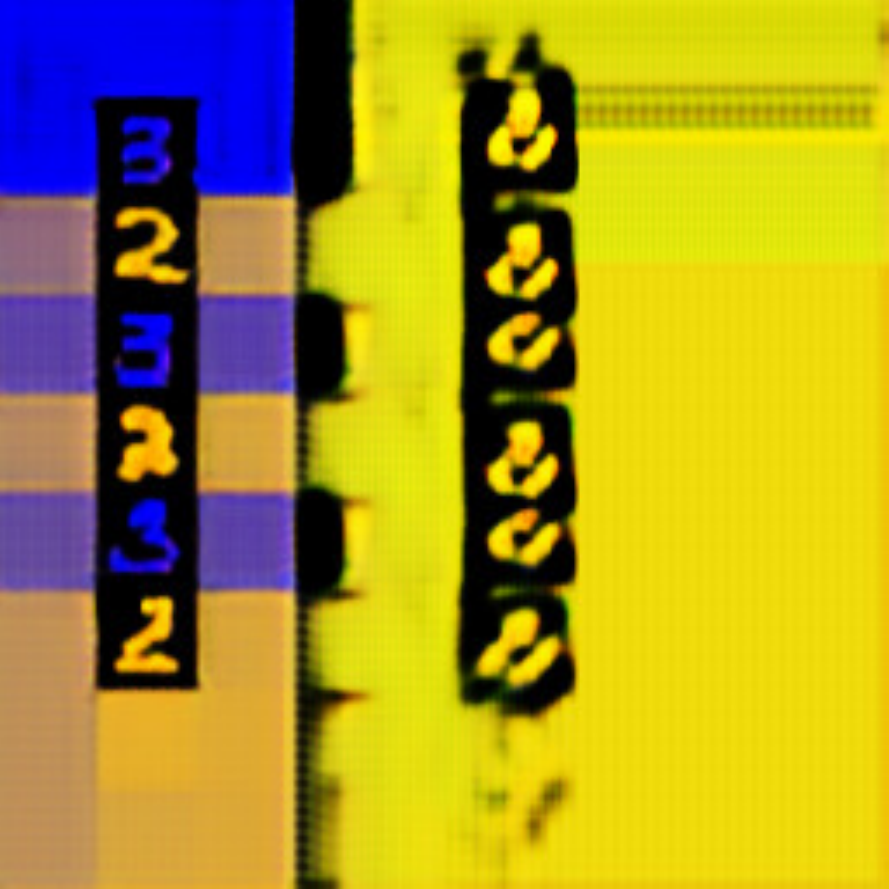} \\
    PS-GM (VED) & \includegraphics[width=0.1\textwidth]{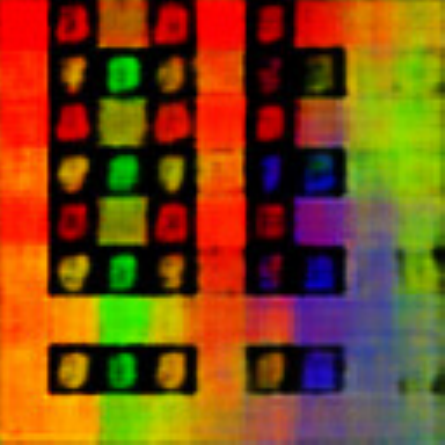}  & \includegraphics[width=0.1\textwidth]{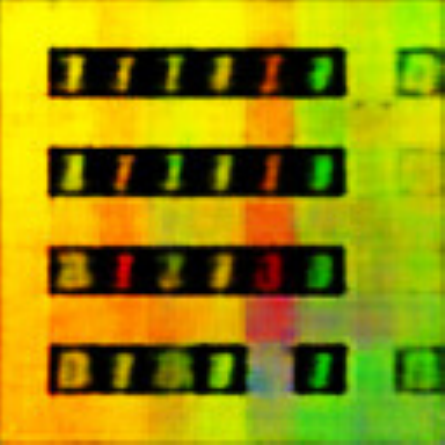} & \includegraphics[width=0.1\textwidth]{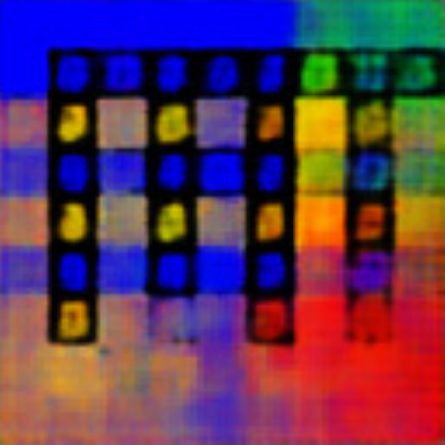} \\
    Baseline (VED) & \includegraphics[width=0.1\textwidth]{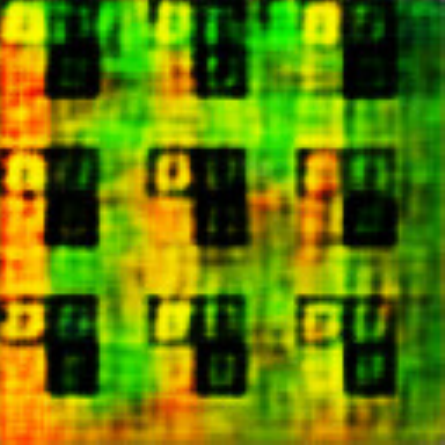}  & \includegraphics[width=0.1\textwidth]{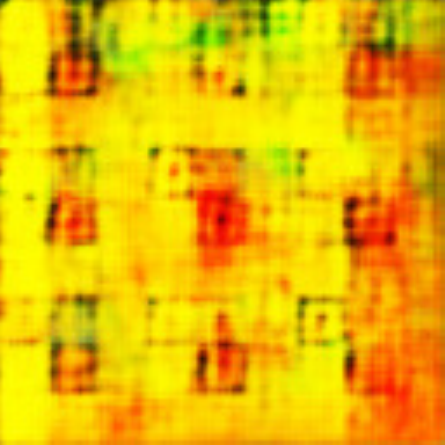} & \includegraphics[width=0.1\textwidth]{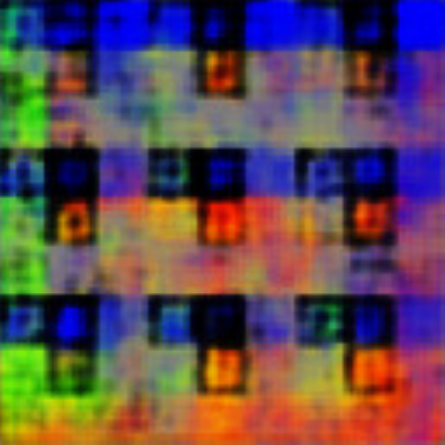} \\
      \end{tabular} 
  \caption{Examples of our image completion pipeline on our synthetic dataset.}
  \label{fig:expsyntheticappendix}
\end{figure*}